\documentclass[11pt]{article}

\usepackage{fullpage,times,url}

\usepackage{amsthm,amsfonts,amsmath,amssymb,epsfig,color,float,graphicx,verbatim,mathtools}
\usepackage{algorithm,algorithmic}

\usepackage{hyperref}
\hypersetup{
	colorlinks   = true, 
	urlcolor     = blue, 
	linkcolor    = blue, 
	citecolor   = black 
}

\newtheorem{theorem}{Theorem}
\newtheorem{proposition}{Proposition}
\newtheorem{lemma}{Lemma}

\newtheorem{remark}{Remark}
\newtheorem{assumption}{Assumption}

\newcommand{\p}[1]{\left( #1 \right)}
\newcommand{\pcc}[1]{\left[ #1 \right]}
\newcommand{\set}[1]{\left\lbrace#1\right\rbrace}
\newcommand{\abs}[1]{\left| #1 \right|}
\newcommand{\reals}{\mathbb{R}}
\newcommand{\E}{\mathbb{E}}

\newcommand{\bx}{\mathbf{x}}

\newcommand{\Ocal}{\mathcal{O}}

\newcommand{\Xcal}{\mathcal{X}}

\newcommand{\norm}[1]{\|#1\|}

\newcommand{\secref}[1]{Sec.~\ref{#1}}

\renewcommand{\eqref}[1]{Eq.~(\ref{#1})}
\newcommand{\lemref}[1]{Lemma~\ref{#1}}

\newcommand{\thmref}[1]{Thm.~\ref{#1}}
\newcommand{\propref}[1]{Proposition~\ref{#1}}

\title{How Good is SGD with Random Shuffling?}
\author{Itay Safran and Ohad Shamir\\
	Weizmann Institute of Science\\
		\texttt{\{itay.safran,ohad.shamir\}@weizmann.ac.il}}
\date{}

\begin{document}

\maketitle 

\begin{abstract}
We study the performance of stochastic gradient descent (SGD) on smooth and 
strongly-convex finite-sum optimization problems. In contrast to the majority 
of existing theoretical works, which assume that individual functions are 
sampled with replacement, we focus here on popular but poorly-understood 
heuristics, which involve going over random permutations of the individual 
functions. This setting has been investigated in several recent works, but the 
optimal error rates remain unclear. In this paper, we provide lower bounds on 
the expected optimization error with these heuristics (using SGD with any 
constant step size), which elucidate their advantages and disadvantages. In 
particular, we prove that after $k$ passes over $n$ individual functions, if 
the functions are re-shuffled after every pass, the best possible optimization 
error for SGD is at least $\Omega\left(1/(nk)^2+1/nk^3\right)$, which partially 
corresponds to recently derived upper bounds. Moreover, if the functions are only shuffled once, then the lower bound 
increases to $\Omega(1/nk^2)$. Since there are strictly smaller upper bounds 
for repeated reshuffling, this proves an inherent performance gap between SGD 
with single shuffling and repeated shuffling. As a more minor contribution, we 
also provide a non-asymptotic $\Omega(1/k^2)$ lower bound (independent of $n$) 
for the incremental gradient method, when no random shuffling takes place. Finally, we provide an indication that our lower bounds are tight, by proving matching upper bounds for univariate quadratic functions.
\end{abstract}

\section{Introduction}

We consider variants of stochastic gradient descent (SGD) for solving unconstrained finite-sum 
problems of the form
\begin{equation}\label{eq:problem}
\min_{\bx\in\Xcal} F(\bx)~=~\frac{1}{n}\sum_{i=1}^{n}f_i(\bx)~,
\end{equation}
where $\Xcal$ is some Euclidean space $\reals^d$ (or more generally some real 
Hilbert space), $F$ is a strongly convex function, and each individual function 
$f_i$ is smooth (with Lipschitz gradients) and Lipschitz on a bounded domain. 
Such problems are extremely common in machine learning applications, which 
often boil down to minimizing the average loss over $n$ data points with 
respect to a class of predictors parameterized by a vector $\bx$. When $n$ is 
large, perhaps the most common approach to solve such problems is via 
stochastic gradient descent, which initializes at some point in $\Xcal$ and 
involves 
iterations of the form $\bx':=\bx-\eta \nabla f_{i}(\bx)$, where $\eta$ is a step size parameter and
$i\in \{1,\ldots,n\}$. The majority of existing theoretical works assume 
that each $i$ is sampled independently across iterations (also known as 
with-replacement sampling). For example, if it is 
chosen independently and uniformly at random from $\{1,\ldots,n\}$, then 
$\E_i[\nabla f_{i}(\bx)|\bx]=\nabla F(\bx)$, so the algorithm can be seen 
as a noisy version of exact gradient descent on $F$ (with iterations of the 
form $\bx':=\bx-\eta\nabla F(\bx)$), which greatly facilitates its analysis. 

However, this straightforward sampling approach suffers from practical 
drawbacks, such 
as requiring truly random data access and hence longer runtime. In practice, it 
is quite common to use \emph{without-replacement} sampling heuristics, which 
utilize the individual functions in some random or even 
deterministic order (see for example 
\cite{bottou2009curiously,bottou2012stochastic,nedic2001convergence,recht2012beneath,
shalev2013stochastic,bertsekas2015convex,feng2012towards}). Moreover, to get 
sufficiently high 
accuracy, it 
is common to perform several passes over the data, where each pass either uses 
the same order as the previous one, or some new random order. The different 
algorithmic variants we study in this paper are presented as Algorithms 
\ref{alg:SGDRR} to 
\ref{alg:SGD} below. We assume that all algorithms take as input the functions $f_1,\ldots,f_n$, a step size parameter $\eta>0$ (which remains constant throughout the iterations), and an initialization point $\bx_0$. The algorithms then perform $k$ passes (which we will also refer to as epochs) over the individual functions, but differ in their sampling strategies:
\begin{itemize}
	\item  Algorithm \ref{alg:SGDRR} (SGD with random reshuffling) chooses a 
	new permutation of the functions at the beginning of every epoch, and 
	processes the 
	individual functions in that order.
	\item Algorithm \ref{alg:SGDRS} (SGD with single shuffling) uses the same 
	random permutation for all $k$ epochs.
	\item Algorithm \ref{alg:SGDC} (usually referred to as the incremental gradient method, see \cite{bertsekas2015convex}) performs $k$ passes over the individual functions, each in the same fixed order (which we will assume without loss of generality to be the canonical order $f_1,\ldots,f_n$)
\end{itemize}
In contrast, Algorithm \ref{alg:SGD} presents SGD using with-replacement 
sampling, where each iteration an individual function is chosen uniformly and 
independently. To facilitate our analysis, we let $\bx_t$ in the pseudocode denote the iterate at the end of epoch $t$.

\begin{minipage}{0.45\textwidth}
	\begin{algorithm}[H]
		\caption{SGD with Random Reshuffling}
		\begin{algorithmic}
			\vskip 0.1cm
			\STATE{$\bx:=\bx_0$}
			\FOR{$t=1,\ldots,k$}
			\STATE{Sample a permutation $\sigma(1),\ldots,\sigma(n)$ of 
			$\{1,\ldots,n\}$ 
				uniformly at random}
			\FOR{$j=1,\ldots,n$}
			\STATE{$\bx:=\bx-\eta\nabla f_{\sigma(j)}(\bx)$}
			\ENDFOR
			\STATE{$\bx_t:=\bx$}
			\ENDFOR
		\end{algorithmic}
		\label{alg:SGDRR}
	\end{algorithm}
\end{minipage}%
\hfill
\begin{minipage}{0.45\textwidth}
	\begin{algorithm}[H]
		\caption{SGD with Single Shuffling}
		\begin{algorithmic}
			\vskip 0.1cm
			\STATE{$\bx:=\bx_0$}
			\STATE{Sample a permutation $\sigma(1),\ldots,\sigma(n)$ of 
			$\{1,\ldots,n\}$ 
				uniformly at random}
			\FOR{$t=1,\ldots,k$}
			\FOR{$j=1,\ldots,n$}
			\STATE{$\bx:=\bx-\eta\nabla f_{\sigma(j)}(\bx)$}
			\ENDFOR
			\STATE{$\bx_t:=\bx$}
			\ENDFOR
		\end{algorithmic}
		\label{alg:SGDRS}
	\end{algorithm}
\end{minipage}

\begin{minipage}{0.45\textwidth}
	\begin{algorithm}[H]
		\caption{Incremental Gradient Method}
		\begin{algorithmic}
			\vskip 0.1cm
			\STATE{$\bx:=\bx_0$}
			\FOR{$t=1,\ldots,k$}
			\FOR{$j=1,\ldots,n$}
			\STATE{$\bx:=\bx-\eta\nabla f_{j}(\bx)$}
			\ENDFOR
			\STATE{$\bx_t:=\bx$}
			\ENDFOR
			\STATE{}	
		\end{algorithmic}
		\label{alg:SGDC}
	\end{algorithm}
\end{minipage}%
\hfill
\begin{minipage}{0.45\textwidth}
	\begin{algorithm}[H]
		\caption{SGD with Replacement}
		\begin{algorithmic}
			\vskip 0.1cm
			\STATE{$\bx:=\bx_0$}
			\FOR{$t=1,\ldots,k$}
			\FOR{$j=1,\ldots,n$}
			\STATE{Sample $i\in \{1, \ldots, n\}$ uniformly}
			\STATE{$\bx:=\bx-\eta\nabla f_i(\bx)$}
			\ENDFOR
			\STATE{$\bx_t:=\bx$}
			\ENDFOR
		\end{algorithmic}
		\label{alg:SGD}
	\end{algorithm}
\end{minipage}
\vskip 0.5cm

\begin{table}
	\begin{center}
		\bgroup
		\def\arraystretch{1.5}
		\begin{tabular}{|c||c|c|c|c|}
			\hline
			&Random Reshuffling & Single Shuffling & Incremental & With 
			Replacement\\\hline\hline
			Upper & $1/k^2$~~\footnotesize \cite{gurbuzbalaban2015random}
			&  $1/k^2$~~\footnotesize \cite{gurbuzbalaban2015convergence} & 
			$1/k^2$
			~~\footnotesize	\cite{gurbuzbalaban2015convergence}& $1/nk$\\
			Bound & $1/n$~~\footnotesize (for $k=1$)~~ 
			\cite{shamir2016without} &
			$1/n$~~\footnotesize (for $k=1$)~~ 
			\cite{shamir2016without}
			&&\\
			& $1/(nk)^2+1/k^3$~~\footnotesize \cite{haochen2018random} & $ \boldsymbol{1/nk^2}$~~\footnotesize (for $1d$ quad.)
			& 
			&\\
			& $1/nk^2$~~\footnotesize \cite{jain2019sgd} & 
			& 
			&\\
			&$ \boldsymbol{1/(nk)^2+1/nk^3}$ ~~\footnotesize (for $1d$ quad.) 
			&
			&
			&\\\hline
			Lower & $1/n$~~\footnotesize (for $k=1$)~~ \cite{haochen2018random} 
			& 
			$\boldsymbol{1/nk^2}$  & 
			$1/k^2$ ~~\footnotesize (\cite{gurbuzbalaban2015convergence}, 
			asymptotic)& 
			$1/nk$\\
			Bound 
			&$\boldsymbol{1/(nk)^2+1/nk^3}$&&$\boldsymbol{1/k^2}$~~\footnotesize(non-asymptotic)&\\\hline
		\end{tabular}
		\egroup
	\end{center}
	
	\caption{Upper and lower bounds on the expected optimization error 
		$\E[F(\bx_k)-\inf_{\bx}F(\bx)]$ for constant-step-size SGD with various 
		sampling 
		strategies, after 
		$k$ passes over $n$ individual functions, in terms of $n,k$. Boldface 
		letters
		refer to new results in this paper. 
		We note that the upper bound of 
		\cite{haochen2018random} additionally requires that the Hessian of each 
		$f_i$ 
		is Lipschitz, and the upper bounds of \cite{haochen2018random} and
		\cite{jain2019sgd} require $k$ to be larger than a problem-dependent 
		parameter 
		(depending for example on the condition number). Also, the upper bound 
		of 
		\cite{shamir2016without} requires functions which are generalized 
		linear 
		functions. Our lower bounds apply under all such assumptions. As to our upper bounds, note that they apply only to univariate quadratic functions. Finally, we note that the upper bound of \cite{jain2019sgd} 
		is 
		actually not on the optimization error for $\bx_k$, but rather on a 
		certain 
		averaging of several iterates -- see Remark \ref{remark:applicability} 
		for a 
		further discussion.}
	\label{table:results}
\end{table}

These without-replacement sampling heuristics are often easier and faster to 
implement in practice. In addition, when using random permutations, they often 
exhibit faster error decay than with-replacement SGD \cite{bottou2009curiously}. A common intuitive 
explanation for this phenomenon is that random permutations force the algorithm 
to touch each individual function exactly once during each epoch, whereas 
with-replacement makes the algorithm touch each function once only in 
expectation. However, theoretically analyzing these sampling heuristics 
has proven to be very challenging, since the individual iterations are no 
longer statistically independent.

In the past few years, some progress has been made in this front, and we 
summarize the known results on the expected optimization error (or at least 
what these results imply\footnote{For example, some of these papers focus on 
bounding 
$\E[\norm{\bx_k-\bx^*}^2]$ where $\bx^*$ is the minimum of $F(\cdot)$, rather 
than the expected optimization error $\E[F(\bx_k)-F(\bx^*)]$. However, for 
strongly convex and smooth 
functions, $\norm{\bx_k-\bx^*}^2$ and $F(\bx_k)-F(\bx^*)$ are the same up to 
the 
strong convexity and smoothness parameters, see for example 
\cite{nesterov2018lectures}.}), as well as our new results, in Table 
\ref{table:results}. First, we note that for SGD with replacement, classical
results imply an optimization error of $\Ocal(1/nk)$ after $nk$ stochastic 
iterations, and this is known to be tight (see for example 
\cite{nemirovski2009robust}). 
For SGD with random reshuffling, better bounds have been shown in recent years, 
generally implying that when the number of epochs $k$ is sufficiently large, 
such sampling schemes are better than with-replacement sampling, with 
optimization error decaying as $1/k^2$ rather than $1/k$. However, 
the optimal dependencies on $n,k$ and other problem-dependent parameters remain 
unclear (HaoChen and Sra \cite{haochen2018random} show that for $k=1$, one 
cannot hope to achieve worst-case error smaller than $\Omega(1/n)$, but for 
$k>1$ not much is known). Some other recent theoretical works on SGD 
with random reshuffling (but under somewhat different settings) include 
\cite{recht2012beneath,ying2018stochastic}. For the incremental gradient method, an $\Ocal(1/k^2)$ 
upper bound was shown in \cite{gurbuzbalaban2015convergence}, as 
well as a matching asymptotic lower bound in terms of $k$. For SGD with single 
shuffling, we are actually not aware of a rigorous theoretical analysis. Thus, 
we only have the $\Ocal(1/k^2)$ upper bound trivially implied by the analysis for the incremental gradient method, 
and for $k=1$, the $\Ocal(1/n)$ upper bound implied by the analysis for random reshuffling 
(since in that case there is no distinction between single shuffling and 
random reshuffling). Indeed, for single shuffling, even 
different epochs are not statistically independent, which makes the analysis 
particularly challenging. 

In this paper, we focus on providing bounds on the expected optimization error of 
SGD with these sampling heuristics, which complement the existing upper bounds 
and provide further insights on the advantages and disadvantages of each. We 
focus on constant-step size SGD, as it simplifies our analysis, and existing 
upper bounds in the literature are derived in the same setting. Our 
contributions are as follows:
\begin{itemize}
	\item For SGD with random reshuffling, we provide in \secref{sec:SGDrandom} a lower bound of $\Omega(1/(nk)^2+1/nk^3)$. Interestingly, it seems to combine the ``best'' behaviors of previous upper bounds: It behaves as 
	$1/n$ for a small constant number $k$ of passes (which is optimal as 
	discussed above), interpolating to $\Ocal(1/(nk)^2)$ when $k$ is large 
	enough, and contains a term decaying cubically with $k$. Moreover, the proof construction applies already for univariate quadratics.
	\item For SGD with a single shuffling, we provide in \secref{sec:single_lbound} a lower bound of $\Omega(1/nk^2)$. Although we are not aware of a previous upper bound to compare 
	to, this lower bound already proves an inherent performance gap compared to 
	random reshuffling: Indeed, in the latter case there is an upper bound of 
	$\Ocal(1/(nk)^2+1/k^3)$, which is smaller than the $\Omega(1/nk^2)$ lower 
	bound for single shuffling when $k$ is sufficiently large. This implies 
	that the added computational effort of repeatedly reshuffling the functions 
	can provably pay off in terms of the optimization error.
	\item For the incremental gradient method, we provide in \secref{sec:cyclic} an $\Omega(1/k^2)$ lower 
	bound. 	We note that a similar bound (at 
	least asymptotically and for a certain $n$) is already implied by  
	\cite[Theorem 3.4]{gurbuzbalaban2015convergence}. Our contribution here is 
	to present a more explicit and non-asymptotic lower bound.
	\item In \secref{sec:upper_bounds}, we provide an indication that our lower bounds are tight, by proving matching upper bounds in the specific setting of univariate quadratic functions. We conjecture that these bounds also hold for multivariate quadratics, and perhaps even to general smooth and strongly convex functions. This is based on our matching lower bounds, as well as the fact that the bounds for with-replacement SGD are known to be tight already for univariate quadratics.
\end{itemize}

We note that in a very recent work (appearing after the initial publication of our work), Rajput et al.\ \cite{rajput2020closing} show an upper bound of $\Ocal(1/nk^3+1/n^2k^2)$ for SGD with random reshuffling for multivariate quadratics, as well as a $\Omega(1/nk^2)$ lower bound for general convex functions. This validates that our lower bounds are tight for quadratics in the random reshuffling case.

\section{Preliminaries}

We let bold-face letters denote vectors. A twice-differentiable function $f$ on 
$\reals^d$ is 
$\lambda$-strongly convex, if its Hessian satisfies $\nabla^2 F(\bx)\succeq 
\lambda I$ for all $\bx$. $f$ is quadratic if it is of the form 
$f(\bx)=\bx^\top A\bx+\mathbf{b}^\top 
\bx+c$ for some matrix $A$, vector $\mathbf{b}$ and scalar $c$. 

We consider finite-sum optimization problems as in \eqref{eq:problem}, and our 
lower bound constructions hold under the following conditions (for some positive parameters $G,\lambda$): 
\begin{assumption}\label{assumption}
$F(\bx)$ is a quadratic finite-sum function of the form 
$\frac{1}{n}\sum_{i=1}^{n}f_i(\bx)$ for some $n>1$, which is 
$\lambda$-strongly convex. Each $f_i$ is convex and quadratic and of the form $f_i(x)=ax^2+bx$, has 
$\lambda$-Lipschitz gradients, and moreover, is 
$G$-Lipschitz 
for any $\bx$ such that 
$\norm{\bx-\bx^*}\leq 1$ where $\bx^*=\arg\min F(\bx)$. Also, the algorithm 
is initialized at some 
$\bx_0$ for which $\norm{\bx_0-\bx^*}\leq 1$. 
\end{assumption}

Before continuing, we make a few remarks about the setting and our results:
\begin{remark}[Constant Condition Number]
	In the above assumption, $\lambda$ plays a double role as both the gradient Lipschitz and strong convexity parameter. This entails that the condition number (defined as the quotient of the two) is constant, hence our lower bounds stem from inherent limitations of each sampling method and not from the constructions being ill-conditioned. We leave the problem of deriving lower bounds for general condition numbers to future work.
\end{remark}

\begin{remark}[Unconstrained Optimization]
For simplicity, in this paper we consider unconstrained SGD, where the 
iterates are not explicitly constrained to lie in some subset of the 
domain. However, we note that existing upper bounds for SGD on strongly convex 
functions often assume an explicit projection on such a subset, in order to 
ensure that the gradients remain bounded. That being said, it is not difficult 
to verify that all our constructions -- which have a very simple structure -- 
are such that the iterates remain in a region with bounded gradients (with 
probability $1$, at least for reasonably small step sizes), in which case 
projections 
will not significantly affect the results. 
\end{remark}

\begin{remark}[Distance from Optimum]
In Assumption \ref{assumption}, we fix the initial distance from 
the optimum to be at most $1$, rather than keeping it as a variable parameter. 
Besides simplifying the constructions, we note that existing SGD upper bounds 
for strongly convex functions often do not explicitly depend on the initial 
distance (both for with-replacement SGD and with random 
reshuffling, see for example
\cite{nemirovski2009robust,rakhlin2012making,jain2019sgd}). Thus, it makes 
sense to study lower 
bounds in which the initial distance is fixed to be some constant.
\end{remark}

\begin{remark}[Applicability of the Lower Bounds]\label{remark:applicability}
We emphasize that in our lower bounds, we focus on (a) SGD with constant step 
size, and (b) the expected performance of the iterate $\bx_k$ after 
exactly $k$ epochs. Thus, they do not formally cover step sizes 
which change across iterations, the performance of other iterates, or the 
performance of some average of the iterates. However, it is not clear that 
these are truly necessary to achieve optimal error bounds in our setting (indeed, many 
existing 
analyses do not require them), and we conjecture that our lower bounds cannot 
be substantially improved even with non-constant step sizes and iterate 
averaging schemes. 
\end{remark}

\section{SGD with Random Reshuffling}\label{sec:SGDrandom}

We begin by discussing SGD with random reshuffling, where at the beginning of 
every epoch we choose a new random order for processing the individual 
functions (Algorithm \ref{alg:SGDRR}). Our main result is the following:

\begin{theorem}\label{thm:main}
For any $k\geq 1, n>1$, and positive $G,\lambda$ such that $G\geq 6\lambda$, 
there exists a function $F$ on $\reals$ and an initialization point $x_0$ 
satisfying Assumption \ref{assumption}, such that for any step size 
$\eta>0$,
\[
\E\left[F(x_k)-\inf_{x}F(x)\right]~\geq~c\cdot \min\left\{\lambda
~,~\frac{G^2}{\lambda}\left(\frac{1}{(nk)^2}+\frac{1}{nk^3}\right)\right\}~,
\]
where $c>0$ is a universal constant. 
\end{theorem}

We remark that the $\lambda$ term seems unavoidable (at least in the univariate setting), as it is a trivial lower bound that holds by Assumption \ref{assumption} for most points in the domain\footnote{e.g.\ for small enough $\lambda$ and when considering a uniform distribution over all points in the domain.}. However, for $nk$ large, this lower bound is
\[
\Omega\left(\frac{G^2}{\lambda}\left(\frac{1}{(nk)^2}+\frac{1}{nk^3}\right)\right)~.
\]
It is useful to compare this bound to the existing optimal bound for SGD with 
replacement, which is
\[
\Theta\left(\frac{G^2}{\lambda nk}\right)
\]
(see for example \cite{rakhlin2012making}). 
First, we note that the $G^2/\lambda$ factor is the same in both of them. The 
dependence on $n,k$ 
though is different: For $k=1$ or constant $k$, our lower bound is 
$\Omega(1/n)$, similar to the with-replacement case, but as $k$ increases, it 
decreases 
cubically (rather than linearly) with $k$. This indicates that even for small 
$k$, random reshuffling is superior to with-replacement sampling, which agrees 
with empirical observations. For $k$ very large 
($k>n$), a phase transition occurs and the bound becomes $1/(nk)^2$ -- that is, 
scaling down quadratically with the total number of individual stochastic 
iterations. That being said, it should be emphasized that $k>n$ is often an 
unrealistic regime, especially in large-scale problems where $n$ is a huge 
number. 

The proof of \thmref{thm:main} appears in \secref{sec:proofthmmain}. It is 
based on a set of very simple constructions, where $F(x)=\frac{\lambda}{2}x^2$, 
and the individual functions are all of the form $f_i(x)=a_i x^2+b_ix$ for 
appropriate $a_i,b_i$. This allows us to write down the iterates 
$x_1,x_2,\ldots$ at the end of each epoch in closed form. The analysis then 
carefully tracks the decay of $\E[x_t^2]$ after each epoch, showing that it 
cannot decay to $0$ too rapidly, hence implying a lower bound on $\E[F(x_k)]$ 
after $k$ epochs.  The main challenge is that unlike SGD with replacement, here 
the stochastic iterations in each epoch are not independent, so computing these 
expectations is not easy. To make it tractable, we identify two distinct 
sources contributing to the error in each epoch: A ``bias'' term, which 
captures the fact that the stochastic gradients at each epoch are statistically 
correlated, hence for a given iterate $\bx$ during the algorithm's run, 
$\E[\nabla f_{\sigma(j)}(\bx)|\bx]\neq \nabla F(\bx)$ (unlike 
the with-replacement 
case where equality holds), and a 
``variance'' term, which captures the inherent noise in the stochastic sampling 
process. For 
different parameter regimes, we use different constructions and focus on either 
the bias or the variance component (which when studied in isolation are more 
tractable), and then combine the various bounds into the final lower bound 
appearing in \thmref{thm:main}. 

We finish with the following remark about a possible extension of the lower 
bound:

\begin{remark}[Convex Functions]
By allowing $\lambda$ to decay to $0$ at a rate governed by $k$ (as well as the remaining problem parameters), we may consider the setting of convex functions which are not necessarily strongly convex (since that for large enough $k$, there exists no $c>0$ such that $\lambda\ge c$). In such a regime, \thmref{thm:main} seems to suggest a lower bound 
(in terms of $n,k$) of 
\[
\Omega\left(G\sqrt{\frac{1}{(nk)^2}+\frac{1}{nk^3}}\right)
~=~
\Omega\left(G\left(\frac{1}{\sqrt{nk^3}}+\frac{1}{nk}\right)\right)~,
\]
since in this scenario we can set $\lambda$ arbitrarily small, and in 
particular as $G\sqrt{1/(nk)^2+1/nk^3}$ so as to maximize the 
lower bound in \thmref{thm:main}. In contrast, \cite{jain2019sgd} shows
a $\Ocal(1/\sqrt{nk})$ upper bound in this setting for SGD with random 
reshuffling, and a similar upper bound hold for SGD with replacement. A similar 
argument can also be applied to the other lower bounds in our paper, extending 
them from the strongly convex to the convex case. However, we emphasize that 
some caution is needed, since our lower bounds do not quantify a dependence on 
the radius of the domain, which is usually explicit in bounds for this setting.
We leave the task of proving a lower bound in the general convex case to future work.
\end{remark}

\section{SGD with a Single Shuffling}\label{sec:single_lbound}

We now turn to the case of SGD where a single random order over the individual 
functions is chosen at the beginning, and the algorithm then cycles over the individual 
functions using that order (Algorithm \ref{alg:SGDRS}). Our main result here is the following:

\begin{theorem}\label{thm:single_shuffling}
	For any $k\geq 1, n>1$, and positive $G,\lambda$ such that $G\geq 
	6\lambda$, 
	there exists a function $F$ on $\reals$ and an initialization point $x_0$ 
	satisfying Assumption \ref{assumption}, such that for any step size 
	$\eta>0$,
	\[
	\E\left[F(x_k)-\inf_{x}F(x)\right]~\geq~c\cdot \min\left\{\lambda
	~,~\frac{G^2}{\lambda nk^2}\right\}~,
	\]
	where $c>0$ is a universal constant. 
\end{theorem}

The proof appears in Subsection \ref{sec:proofthmsingle}. In the single shuffling case, we are not aware of a previously known upper bound to compare to (except the $\Ocal(1/k^2)$ bound for the incremental gradient method below, which trivially applies also to SGD 
with single shuffling). However, the lower bound already implies an interesting 
separation between single shuffling and random reshuffling: In the former case, 
$\Omega(1/nk^2)$ is the best we can hope to achieve, whereas in the latter 
case, we have seen upper bounds which are strictly better when $k$ is 
sufficiently large (i.e., $\Ocal(1/(nk)^2)$). To the best of our knowledge, 
this is the first formal separation between these two shuffling schemes for 
SGD: It implies that the added computational effort of repeatedly reshuffling 
the functions can provably pay off in terms of the optimization error. It would 
be quite interesting to understand whether this separation might also occur for 
smaller values of $k$ as well, which is definitely true if our 
$\Omega(1/(nk)^2+1/nk^3)$ lower bound for random reshuffling is tight. It would 
also be interesting to derive a good upper bound for SGD with single shuffling, 
which is a common heuristic (indeed, we prove such a bound in \secref{sec:upper_bounds}, but only for univariate quadratics). 

\section{Incremental Gradient Method}\label{sec:cyclic}

Next, we turn to discuss the incremental gradient method, where the individual functions are cycled over in a fixed deterministic order. We note that for this algorithm, an 
$\Omega(1/k^2)$ lower bound was already proven in 
\cite{gurbuzbalaban2015convergence}, but in 
an asymptotic form, and only for $n=2$. Our 
contribution here is to provide an explicit, non-asymptotic bound:

\begin{theorem}\label{thm:cyclic}
For any $k\geq 1, n>1$, and positive $G,\lambda$ such that $G\geq 6\lambda$, 
there exists a function $F$ on $\reals$ and an initialization point $x_0$ 
satisfying Assumption \ref{assumption}, such that if we run the incremental gradient method for $k$ 
epochs with any step size $\eta>0$, then
\[
F(x_k)-\inf_{x}F(x)~\geq~ c\cdot\min\left\{\lambda~,~\frac{G^2}{\lambda 
k^2}\right\}
\]
where $c>0$ is a universal constant.
\end{theorem}

The proof (which follows a strategy broadly similar to \thmref{thm:main}) 
appears in \secref{sec:proofthmcyclic}. Comparing this theorem with our other 
lower bounds and the associated upper bounds, 
it is clear that there is a high price to pay (in a worst-case sense) for using 
a fixed, non-random order, as the bound does not improve at all with more 
individual functions $n$. Indeed, recalling that the bound for 
with-replacement SGD is $\Ocal(G^2/\lambda nk)$, it follows that incremental gradient method can 
beat with-replacement SGD only when $\frac{G^2}{\lambda 
k^2}\leq\frac{G^2}{\lambda nk}$, or $k\geq n$. For large-scale problems where 
$n$ is big, this is often an unrealistically large value of $k$.

\section{Tight Upper Bounds for One-Dimensional Quadratics}\label{sec:upper_bounds}

As discussed in the introduction, for SGD with random reshuffling and single shuffling, there is a gap between the lower bounds we present here, and known upper bounds in the literature. In this section, we provide an indication that our lower bounds are tight, by proving matching upper bounds (up to log factors) for the setting of univariate quadratic functions\footnote{I.e., $x\mapsto ax^2+bx$. Note that for simplicity, we assume no constant term $c$ as in  $ax^2+bx+c$, as it plays no role in the optimization process.}. Although this is a special case, we note that the standard $\Theta(1/nk)$ bounds for SGD with replacement on strongly convex functions are known to be tight already for univariate quadratics. This leads us to conjecture that even for without-replacement sampling schemes, the optimal rates for univariate quadratics are also the optimal rates for general strongly convex functions. 

Before stating our upper bounds, we make the following assumption on the target functions $ f_i $:
\begin{assumption}\label{assumption_ubound}
	$ F(x) = \frac{1}{n}\sum_{i=1}^{n}f_i(x) $ is $ \lambda $-strongly convex. Moreover, each $ f_i(x)=\frac{a_i}{2}x^2 - b_ix $ is convex, has $ L $-Lipschitz gradients, and satisfies $ \abs{f_i^{\prime}(x^*)}\le G $ where $x^*=\arg\min_x F(x)$. 
\end{assumption}
For the single shuffling case we have the following theorem:
\begin{theorem}\label{thm:single_ubound}
	Let $ F(x)\coloneqq\frac{\lambda}{2}x^2 - bx=\frac{1}{n}\sum_{i=1}^{n}f_i(x) $, where $ f_i(x)=\frac{1}{2}a_ix^2-b_ix $ satisfy Assumption \ref{assumption_ubound}, and assume that $ \frac{L}{\lambda} \le \frac{nk}{\log\p{n^{0.5}k}} $. Then single shuffling SGD with a fixed step size of $ \eta = \frac{\log\p{n^{0.5}k}}{\lambda nk} $ satisfies\footnote{Letting $\kappa\coloneqq L/\lambda$ denote the condition number, the second term in the right hand side can equivalently be written as $\frac{G^2\kappa^2}{\lambda nk^3}$.}
	\[
		\E\left[F(x_k)-\inf_xF(x)\right] \le \tilde{\Ocal}\p{\frac{\lambda}{nk^2}\p{x_0-x^*}^2 + \frac{G^2L^2}{\lambda^3 nk^2}},
	\]
	where the expectation is taken over drawing a permutation $ \sigma:[n]\to[n] $ uniformly at random, and the big O tilde notation hides a universal constant and factors poly-logarithmic in $ n $ and $ k $.
\end{theorem}
For SGD with random reshuffling, we present the following theorem:
\begin{theorem}\label{thm:reshuffle_ubound}
	Let $ F(x)\coloneqq\frac{\lambda}{2}x^2-bx=\frac{1}{n}\sum_{i=1}^{n}f_i(x) $, where $ f_i(x)=\frac{1}{2}a_ix^2-b_ix $ satisfy Assumption \ref{assumption_ubound}, and assume that $ \frac{L}{\lambda} \le \frac{k}{2\log(nk)} $. Then random shuffling SGD with a fixed step size of $ \eta = \frac{\log(nk)}{\lambda nk} $ satisfies\footnote{Similarly to the above footnote, the second term in the right hand side can equivalently be written as $\frac{G^2\kappa^2}{\lambda}\p{\frac{1}{n^2k^2} + \frac{1}{nk^3}}$.}
	\[
		\E\left[F(x_k)-\inf_xF(x)\right] \le \tilde{\Ocal}\p{\frac{\lambda}{n^2k^2}\p{x_0-x^*}^2 + \frac{G^2L^2}{\lambda^3}\p{\frac{1}{n^2k^2} + \frac{1}{nk^3}}},
	\]
	where the expectation is taken over drawing $ k $ permutations $ \sigma_i:[n]\to[n] $ uniformly at random, and the big O tilde notation hides a universal constant and factors poly-logarithmic in $ n $ and $ k $.
\end{theorem}

The formal proofs appear in \secref{sec:proofs}.

It is easy to verify that these upper bounds match our lower bounds in Theorems \ref{thm:main} and \ref{thm:single_shuffling} in terms of the dependence on $n,k$. Moreover, our requirement of $k\ge \tilde{\Omega}(\kappa)$ (recall that $\kappa\coloneqq L/\lambda$) for random reshuffling is also made in \cite{haochen2018random}. As to the other parameters, it is important to note that our lower bound constructions (which also utilize univariate quadratics) are in a regime where both $L/\lambda$ and $x_0-x^*$ are constants, and they match the upper bounds in this case. In particular, \thmref{thm:single_ubound} then reduces to $\tilde{\Ocal}\left(\frac{\lambda}{nk^2}+\frac{G^2}{\lambda nk^2}\right)$, which is $\tilde{\Ocal}(\frac{G^2}{\lambda nk^2})$ under the assumption $G\geq 6\lambda$ which we make in the lower bound. Similarly,  \thmref{thm:reshuffle_ubound} reduces to 
\[
\tilde{\Ocal}\left(\frac{\lambda}{n^2k^2}+\frac{G^2}{\lambda}\left(\frac{1}{n^2k^2}+\frac{1}{nk^3}\right)\right)~=~
\tilde{\Ocal}\left(\frac{G^2}{\lambda}\left(\frac{1}{n^2k^2}+\frac{1}{nk^3}\right)\right)
\]
if $G\geq 6\lambda$. We leave the problem of getting matching upper and lower bounds in all parameter regimes of $G,L,\lambda$ to future work.

While the assumption of univariate quadratics is restrictive, our main purpose here is to indicate the potential tightness of our lower bounds, and elucidate how without-replacement sampling can lead to faster convergence in a simple setting. Our proof is based on evaluating a closed-form expression for the iterate at the $ k $-th epoch, splitting deteministic and stochastic terms, and then carefully bounding the stochastic terms using a Hoeffding-Serfling type inequality and the deterministic term using the AM-GM inequality.

We conjecture that our upper bounds can be generalized to general quadratic functions, and perhaps even to general smooth and strongly convex functions. The main technical barrier is that our proof crucially uses the commutativity of the scalar-valued $a_i$'s. Once we deal with matrices, we essentially require (a special case of) a matrix-valued arithmetic-geometric mean inequality studied in \cite{recht2012beneath} (See \eqref{eq:amgm} for the part of the proof where we require this inequality). Unfortunately, as of today this conjectured inequality is not known to hold except in extremely special cases. 
%

\section{Proofs}\label{sec:proofs}

\subsection{Proof of \thmref{thm:main}}\label{sec:proofthmmain}

For simplicity, we will prove the theorem assuming the number of components $n$ 
in our function is an even number. This is without loss of generality, since if 
$n>1$ is odd, let 
$F_{n-1}(\bx)=\frac{1}{n-1}\sum_{i=1}^{n-1}f_i(\bx)$ be the function achieving 
the lower bound using an even number $n-1$ of components, and define 
$F(\bx)=\frac{1}{n}\left(\sum_{i=1}^{n-1}f_i(\bx)+f_{n}(\bx)\right)$ where 
$f_n(\bx):=0$. $F()$ has the same Lipschitz parameter $G$ as $F_{n-1}()$, and a 
strong convexity parameter $\lambda$ smaller than that of $F_{n-1}()$ by a 
$\frac{n}{n+1}$ factor which is always in $[\frac{3}{4},1]$. Moreover, it is 
easy to see that for a fixed step size, the distribution of the iterates after 
$k$ epochs is the 
same over $F()$ and $F_{n-1}()$, since SGD does not move on any iteration 
where $f_n$ is chosen. Therefore, the lower bound on $F_{n-1}$ translates to a 
lower bound on $F()$ up to a small factor which can be absorbed into the 
numerical constants. Thus, in what follows, we will assume that $n$ is even and 
that $G\geq 4\lambda$, whereas in the theorem statement we make the slightly 
stronger assumption $G\geq 
6\lambda$ so that the reduction described above will be valid. 

The proof of the theorem is based on the following three propositions, each 
using a somewhat different construction and analysis:

\begin{proposition}\label{prop:lowerbound1}
For any even $n$ and any positive $G,\lambda$ such that $G\geq 2\lambda$, there 
exists a function $F$ on $\reals$ satisfying Assumption \ref{assumption}, such 
that 
for any step 
size $\eta>0$,
\[
\E\left[F(x_k)-\inf_{x}F(x)\right]~\geq~c\cdot 
\min\left\{\lambda~,~\frac{G^2}{\lambda nk^3}\right\}
\]
where $c>0$ is a universal constant. 
\end{proposition}

\begin{proposition}\label{prop:lowerbound2}
Suppose that $k\geq n$ and that $n$ is even. For any positive $G,\lambda$ such 
that $G\geq 2\lambda$, there 
exists a function $F$ on 
$\reals$ satisfying Assumption \ref{assumption}, such that for any step size 
$\eta\geq \frac{1}{100\lambda n^2}$, 
\[
\E\left[F(x_k)-\inf_{x}F(x)\right]~\geq~c\cdot\frac{G^2}{\lambda (nk)^2}
\]
where $c>0$ is a numerical constant. 
\end{proposition}

\begin{proposition}\label{prop:lowerbound3}
Suppose $k>1$ and that $n$ is even. For any positive $G,\lambda$ such that 
$G\geq 4\lambda$, there 
exists a 
function $F$ on 
$\reals$ satisfying Assumption \ref{assumption}, such that for any step size 
$\eta\leq \frac{1}{100\lambda n^2}$, 
\[
\E\left[F(x_k)-\inf_{x}F(x)\right]~\geq~c\cdot\min\left\{\lambda~,~\frac{G^2}{\lambda(nk)^2}\right\}
\]
where $c>0$ is a numerical constant. 
\end{proposition}

The proof of each proposition appears below, but let us first show how 
combining these implies our theorem. We consider two cases:
\begin{itemize}
\item If $k\leq n$, then $\frac{1}{nk^3}\geq \frac{1}{(nk)^2}$, so by 
Proposition \ref{prop:lowerbound1},
\[
\E\left[F(x_k)-\inf_{x}F(x)\right]~\geq~ 
c\cdot\min\left\{\lambda~,~\frac{G^2}{\lambda nk^3}\right\}
~\geq~ 
c\cdot\min\left\{\lambda~,~\frac{G^2}{2\lambda}\left(\frac{1}{(nk)^2}+\frac{1}{nk^3}
\right)\right\}~.
\]
\item If $k\geq n$ (which implies $k>1$ since $n$ is even), we have 
$\frac{1}{nk^3}\leq \frac{1}{(nk)^2}$, and by combining 
Proposition \ref{prop:lowerbound2} and Proposition \ref{prop:lowerbound3} 
(which together cover any positive step size), 
\[
\E\left[F(x_k)-\inf_{x}F(x)\right]~\geq~c\cdot\min\left\{\lambda~,~\frac{G^2}{\lambda(nk)^2}\right\}
~\geq~c\cdot\min\left\{\lambda~,~\frac{G^2}{2\lambda}\left(\frac{1}{(nk)^2}+\frac{1}{nk^3}\right)\right\}
\]
\end{itemize}
Thus, in any case we get $\E\left[F(x_k)-\inf_{x}F(x)\right]\geq 
c\cdot\min\left\{\lambda~,~\frac{G^2}{2\lambda}
\left(\frac{1}{(nk)^2}+\frac{1}{nk^3}\right)\right\}$, from which the result 
follows. We remark that to combine Propositions \ref{prop:lowerbound2} and \ref{prop:lowerbound3}, one must combine the two constructions in the propositions using a bivariate function, where each dimension utilizes a different proposition in an orthogonal direction. This way we are guaranteed to obtain the worst convergence rate of each proposition, resulting in the desired lower bound.

\subsubsection{Proof of Proposition \ref{prop:lowerbound1}}

We will need the following key technical lemma, whose proof (which is rather 
long and technical) appears in Appendix \ref{sec:proofvarterm}:
\begin{lemma}\label{lem:varterm}
	Let $\sigma_0,\ldots,\sigma_{n-1}$ (for even $n$) be a random permutation 
	of $(1,1,\ldots,1,-1,-1,\ldots,-1)$ (where both $1$ and $-1$ appear exactly 
	$n/2$ times). Then there is a numerical constant $c>0$, such that for any 
	$\alpha>0$,
	\[
	\E\left[\left(\sum_{i=0}^{n-1}\sigma_i (1-\alpha)^{i}\right)^2\right]~\geq~
	c\cdot \min\left\{1+\frac{1}{\alpha}~,~n^3\alpha^2\right\}
	\]
\end{lemma}
Let $G,\lambda,n$ be fixed (assuming 
$G\geq 
2\lambda$ and $n$ is even). We will 
use the following function:
\[
F(x)=\frac{1}{n}\sum_{i=1}^{n}f_i(x)~=~\frac{\lambda}{2}x^2~,
\]
where $\inf_{x}F(x)=0$, and
\begin{equation}\label{eq:prop1construction}
f_i(x)=\begin{cases} \frac{\lambda}{2}x^2+\frac{G}{2}x& i\leq \frac{n}{2}\\
\frac{\lambda}{2}x^2-\frac{G}{2}x& i> \frac{n}{2}\end{cases}~.
\end{equation}
Also, we assume that the algorithm is initialized at $x_0=1$.
On this function, we have that during any 
single epoch, we 
perform $n$ iterations of the form
\[
x_{new}=(1-\eta\lambda)x_{old}+\frac{\eta G}{2}\sigma_i,
\]
where $\sigma_0,\ldots,\sigma_{n-1}$ are a random permutation of $\frac{n}{2}$ 
$1$'s and $\frac{n}{2}$ $-1$'s. Repeatedly applying this inequality, we get 
that after $n$ iterations, the relationship between the first and last iterates 
in the epoch satisfy
\begin{align}
x_{t+1}~&=~(1-\eta\lambda)^n x_t+\frac{\eta 
G}{2}\sum_{i=0}^{n-1}\sigma_i(1-\eta\lambda)^{n-i-1}\nonumber\\
&= ~(1-\eta\lambda)^n x_t+\frac{\eta 
	G}{2}\sum_{i=0}^{n-1}\sigma_i(1-\eta\lambda)^{i}~\label{eq:single_epoch}.
\end{align}
(in the last equality, we used the fact that $\sigma_1,\ldots,\sigma_n$ are 
exchangeable). Using this and the fact that $\E[\sigma_i]=0$, we get that
\begin{equation}\label{eq:recurse_of_1}
\E[x_{t+1}^2]~=~(1-\eta\lambda)^{2n}\E[x_t^2]+\left(\frac{\eta G}{2}\right)^2
\cdot\beta_{n,\eta,\lambda}~,
\end{equation}
where
\begin{equation}\label{eq:betadef}
\beta_{n,\eta,\lambda} ~:=~ 
\E\left[\left(\sum_{i=0}^{n-1}\sigma_i(1-\lambda\eta)^{n-i-1}\right)^2\right]~=~
\E\left[\left(\sum_{i=0}^{n-1}\sigma_i(1-\lambda\eta)^{i}\right)^2\right].
\end{equation}
Note that if $\lambda\eta\geq 1$, then by \lemref{lem:varterm}, 
$\beta_{n,\eta,\lambda}\geq c$ for some positive constant $c$, and we get that
\[
\E[x_{t+1}^2]~\geq~\left(\frac{\eta G}{2}\right)^2\cdot c~\geq~
\left(\frac{G}{2\lambda}\right)^2\cdot c
\]
for all $t$, and therefore $\E[F(x_k)]=\frac{\lambda}{2}\E[x_k^2]\geq 
c\frac{G^2}{8\lambda}\geq c\frac{G^2}{8\lambda nk^3}$, so the proposition we 
wish to prove holds. Thus, we will assume from now on that $\lambda\eta <1$.

With this assumption, repeatedly applying \eqref{eq:recurse_of_1} and 
recalling that $x_0=1$, we have
\begin{align}
\E[x_k^2]~&\geq~ (1-\eta\lambda)^{2nk}+\left(\frac{\eta 
G}{2}\right)^2\cdot\beta_{n,\eta,\lambda}\sum_{t=0}^{k-1}(1-\eta\lambda)^{2nt}\notag\\
&=~ (1-\eta\lambda)^{2nk}+\left(\frac{\eta 
	G}{2}\right)^2\cdot\beta_{n,\eta,\lambda}\cdot\frac{1-(1-\eta\lambda)^{2nk}}{
1-(1-\eta\lambda)^{2n}}~.\label{eq:recursek}
\end{align}
We now consider a few cases (recalling that the case $\eta\lambda\geq 1$ was 
already treated earlier):
\begin{itemize}
\item If $\eta\lambda \leq \frac{1}{2nk}$, then we have
\[
\E[x_k^2]~\geq~ (1-\eta\lambda 
)^{2nk}~\geq~\left(1-\frac{1}{2nk}\right)^{2nk}\geq 
\frac{1}{4}
\]
for all $n,k$.
\item If $\eta\lambda \in \left(\frac{1}{2nk},\frac{1}{2n}\right)$ then by 
Bernoulli's inequality, we have $1\geq(1-\eta\lambda)^{2n}\geq 
1-2n\eta\lambda>0$, 
and therefore, by \eqref{eq:recursek}
\[
\E[x_k^2]~\geq~ 
\frac{\eta^2
G^2\beta_{n,\eta,\lambda}(1-(1-1/2nk)^{2nk})}{4(1-(1-2n\eta\lambda))}
~\geq~
\frac{\eta 
G^2\beta_{n,\eta,\lambda}(1-\exp(-1))}{8\lambda 
n}~.
\]
Plugging in \lemref{lem:varterm} and simplifying a bit, this is at 
least
\[
\frac{c\eta G^2}{\lambda n}\cdot 
\min\left\{\frac{1}{\eta\lambda},n^3(\eta\lambda)^2\right\}~=~
\frac{c\eta G^2}{\lambda n}\cdot n^3(\eta\lambda)^2~=~
c\eta^3\lambda n^2 G^2
\]
for some numerical constant $c>0$. Using the assumption that 
$\eta\lambda \geq \frac{1}{2nk}$ (which implies $\eta\geq \frac{1}{2\lambda 
nk}$), this is at least 
\[
\frac{c}{8}\cdot\frac{G^2}{\lambda^2 nk^3}~.
\]
\item If $\eta\lambda\in \left[\frac{1}{2n},1\right)$, then 
$\frac{1-(1-\eta\lambda)^{2nk}}{1-(1-\eta\lambda)^{2n}}$ is at least some 
numerical constant $c>0$, so \eqref{eq:recursek} implies
\[
\E[x_k^2]~\geq~ c\left(\frac{\eta 
	G}{2}\right)^2\cdot\beta_{n,\eta,\lambda}~.
\]
By \lemref{lem:varterm}, this is at least
\[
c'\left(\frac{\eta 
	G}{2}\right)^2\cdot\min\left\{1+\frac{1}{\eta\lambda}~,~n^3(\eta\lambda)^2\right\}
~=~c'\left(\frac{\eta 
	G}{2}\right)^2\left(1+\frac{1}{\eta\lambda}\right)~\geq~
\frac{c'\eta G^2}{4\lambda}
\]
Since $\eta\geq \frac{1}{2\lambda n}$, this is at 
least
\[
\frac{c'G^2}{8\lambda^2 n}~\geq~\frac{c'G^2}{8\lambda^2 nk^3}~. 
\]
\end{itemize}
Combining all the cases, we get overall that
\[
\E[x_k^2]~\geq~ c\cdot \min\left\{1,\frac{G^2}{\lambda^2 nk^3}\right\}
\]
for some numerical constant $c>0$. Noting that 
$\E[F(x_k)]=\E\left[\frac{\lambda}{2}x_k^2\right]=\frac{\lambda}{2}\E\left[x_k^2\right]$
 and combining with the above, the result follows.

\subsubsection{Proof of Proposition \ref{prop:lowerbound2}}

We use the same construction as in the proof of Proposition 
\ref{prop:lowerbound1}, where $F(x)=\frac{\lambda}{2}x^2$, and leading to 
\eqref{eq:recursek}, namely
\begin{equation}\label{eq:recursek2}
\E[x_k^2]~\geq~ (1-\eta\lambda)^{2nk}+\left(\frac{\eta 
	G}{2}\right)^2\cdot\beta_{n,\eta,\lambda}\cdot\frac{1-(1-\eta\lambda)^{2nk}}{
	1-(1-\eta\lambda)^{2n}}~,
\end{equation}
where $\beta_{n,\eta,\lambda}=\E\left[\left(\sum_{i=0}^{n-1}\sigma_i 
(1-\lambda\eta)^i\right)^2\right]$, $\sigma_0,\ldots,\sigma_n$ are a random 
permutation of $\frac{n}{2}$ $1$'s and $\frac{n}{2}$ $-1$'s.

As in the proof of Proposition \ref{prop:lowerbound1}, we consider several 
regimes of $\eta\lambda$. In the same manner as in that proof, it is easy to 
verify that when $\eta\lambda>1$ or $\eta\lambda \leq \frac{1}{2nk}$, then 
$\E[x_k^2]$ is at least a positive constant (hence $\E[F(x_k)]\geq 
\Omega(\lambda)$) , and when $\eta\lambda \in \left[\frac{1}{2n},1\right)$, 
$\E[x_k^2]\geq \frac{c'G^2}{2\lambda^2 n}$ for a numerical constant $c'>0$ 
(hence $\E[F(x_k)]\geq \Omega(G^2/\lambda n)$). In both these cases, the 
statement in our proposition follows, so it is enough to consider the regime 
$\eta\lambda\in \left(\frac{1}{2nk},\frac{1}{2n}\right)$. 

In this regime, by Bernoulli's inequality, we have $0< 
1-(1-\eta\lambda)^{2n}\leq 1-(1-2n\eta\lambda)=2n\eta\lambda$, so we
can lower bound \eqref{eq:recursek2} by
\[
\left(\frac{\eta 
	G}{2}\right)^2\cdot\beta_{n,\eta,\lambda}\frac{1-(1-\eta\lambda)^{2nk}}{2n\eta\lambda}~=~
\frac{\eta G^2\beta_{n,\eta,\lambda}(1-(1-\eta\lambda)^{2nk})}{8\lambda n}~.
\]
Since we assume $\eta\lambda\geq \frac{1}{2nk}$, it follows that 
$1-(1-\eta\lambda)^{2nk}\geq 1-(1-1/2nk)^{2nk}\geq c$ for some positive $c>0$. 
Plugging this and the
bound for $\beta_{n,\eta,\lambda}$ from \lemref{lem:varterm}, the displayed 
equation above is at least
\[
\frac{c\eta G^2}{8\lambda n} 
\cdot\min\left\{\frac{1}{\eta\lambda},n^3(\eta\lambda)^2\right\}~=~
\frac{c\eta G^2}{8\lambda n} 
\cdot n^3(\eta\lambda)^2~=~
\frac{c}{8}G^2\lambda 
\eta^3n^2~.
\]
Since we assume $\eta\geq \frac{1}{100\lambda 
n^2}$, this is at least
\[
c'\cdot\frac{G^2}{\lambda^2 
n^4}
\]
for some numerical $c'>0$. Since we assume that $k\geq n$, this is at least 
$c'\cdot\frac{G^2}{\lambda^2 (nk)^2}$. Noting that 
$\E[F(x_k)]=\E\left[\frac{\lambda}{2}x_k^2\right]=\frac{\lambda}{2}\E\left[x_k^2\right]$
 and combining with the above, the result follows.

\subsubsection{Proof of Proposition \ref{prop:lowerbound3}}

To simplify some of the notation, we will prove the 
result for a function which is $\lambda/2$-strongly convex (rather than 
$\lambda$-strongly convex), assuming $G\geq 2\lambda$, and notice that this 
only 
affects the universal 
constant $c$ in the bound. Specifically, we use the following function:
\[
F(x)=\frac{1}{n}\sum_{i=1}^{n}f_i(x)~=~\frac{\lambda}{4}x^2~,
\]
where $\inf_{x}F(x)=0$, and
\[
f_i(x)=\begin{cases} \frac{\lambda}{2} x^2+\frac{G}{2}x& i\leq \frac{n}{2}\\
-\frac{G}{2}x& i> \frac{n}{2}\end{cases}~.
\]
Also, we assume that the algorithm is initialized at $x_0=-1$.
On this function, we have that during any 
single epoch, we 
perform $n$ iterations of the form
\[
x_{new}=(1-\eta\lambda\sigma_i)x_{old}+\frac{\eta G}{2}(1-2\sigma_i),
\]
where $\sigma_0,\ldots,\sigma_{n-1}$ are a random permutation of $\frac{n}{2}$ 
$1$'s and $\frac{n}{2}$ $0$'s. Repeatedly applying this equation, we get 
that after $n$ iterations, the relationship between the iterates $x_t$ and 
$x_{t+1}$ is
\begin{equation}\label{eq:recursionform}
x_{t+1}~=~x_t\cdot\prod_{i=0}^{n-1}(1-\eta\lambda\sigma_i) +\frac{\eta 
G}{2}\sum_{i=0}^{n-1}(1-2\sigma_i)\prod_{j=i+1}^{n-1}(1-\eta\lambda\sigma_j)
\end{equation}
As a result, and using the fact that $\sigma_1,\ldots,\sigma_n$ are independent 
of $x_t$ and in $\{0,1\}$, we have
\begin{align}
\E[x_{t+1}^2]~&\geq~ 
\E\left[x_t^2\cdot\prod_{i=0}^{n-1}(1-\eta\lambda\sigma_i)^2\right]+\eta G\cdot
\E\left[x_t\left(\prod_{i=0}^{n-1}(1-\eta\lambda\sigma_i)\right)
\left(\sum_{i=0}^{n-1}(1-2\sigma_i)\prod_{j=i+1}^{n-1}(1-\eta\lambda\sigma_j)\right)
\right]\notag\\
&\geq~
(1-\eta\lambda)^{n}\cdot \E[x_t^2]+\eta G\cdot
\E[x_t]\cdot \E\left[\left(\prod_{i=0}^{n-1}(1-\eta\lambda\sigma_i)\right)
\left(\sum_{i=0}^{n-1}(1-2\sigma_i)\prod_{j=i+1}^{n-1}(1-\eta\lambda\sigma_j)\right)
\right]\label{eq:recurse2}
\end{align}
We now wish to use \lemref{lem:prodtosum} from Appendix \ref{sec:technical}, in 
order to replace the products in 
the expression above by sums. To that end, and in order to simplify the 
notation, define
\begin{equation}\label{eq:AB}
A:=\prod_{i=0}^{n-1}(1-\eta\lambda\sigma_i)~~,~~B_{i}:=\prod_{j=i+1}^{n-1}(1-\eta\lambda
 \sigma_j)~~,~~\tilde{A}:=1-\eta\lambda\sum_{i=1}^{n}\sigma_i=1-\frac{\eta\lambda
  n}{2}~~,~~\tilde{B}_i:=1-\eta\lambda\sum_{j=i+1}^{n}\sigma_j~,
\end{equation}
and note that by \lemref{lem:prodtosum},
\begin{equation}\label{eq:ABprod}
A\sum_{i=0}^{n-1}(1-2\sigma_i)B_i ~\leq~ 
\left(\tilde{A}\pm 2\left(\eta\lambda\sum_{i=0}^{n-1}\sigma_i\right)^2\right)
\left(\sum_{i=0}^{n-1}(1-2\sigma_i)\tilde{B}_i\pm 
2\sum_{i=0}^{n-1}\left(\eta\lambda
\sum_{j=i+1}^{n-1}\sigma_j\right)^2\right)~,
\end{equation}
where $\pm$ is taken to be either plus or minus depending on the 
sign of $\tilde{A}$ and $\sum_{i=0}^{n-1}(1-2\sigma_i)\tilde{B}_i$, to make the 
inequality valid (we note that eventually we will show that these terms are 
relatively negligible). Opening the product, and using the deterministic upper 
bounds
\begin{equation}\label{eq:detupbound1}
|\tilde{A}|\leq 1~~,~~
\left(\eta\lambda\sum_{i=0}^{n-1}\sigma_i\right)^2\leq 
(\eta\lambda n)^2
\end{equation}
and
\begin{equation}\label{eq:detupbound2}
\left|\sum_{i=0}^{n-1}(1-2\sigma_i)\tilde{B}_i\right|\leq n~~,~~
\sum_{i=0}^{n-1}\left(\eta\lambda
\sum_{j=i+1}^{n-1}\sigma_j\right)^2~\leq~ n(\eta\lambda n)^2~\leq 
\frac{1}{10^4 n}~,
\end{equation}
(which follow from the assumption that $\eta\leq \frac{1}{100\lambda n^2}$), we 
can upper 
bound \eqref{eq:ABprod} by
\[
\tilde{A}\sum_{i=0}^{n-1}(1-2\sigma_i)\tilde{B}_i+2(\eta\lambda 
n)^2\cdot 
\left(n+\frac{2}{100n}\right)
+n(\eta\lambda n)^2 ~\stackrel{(*)}{\leq}~ 
\tilde{A}\sum_{i=0}^{n-1}(1-2\sigma_i)\tilde{B}_i+\frac{301}{100}(\eta\lambda)^2n^3~,
\]
where in $(*)$ we used the fact that $n\geq 2$ and therefore 
$n+\frac{2}{100n}\leq n+\frac{1}{100}\leq (1+\frac{1}{200})n$. 
Substituting 
back the definitions 
of $\tilde{A},\tilde{B}$ and plugging back into \eqref{eq:ABprod}, we get 
that 
\begin{align*}
\E&\left[\left(\prod_{i=0}^{n-1}(1-\eta\lambda\sigma_i)\right)
\left(\sum_{i=0}^{n-1}(1-2\sigma_i)\prod_{j=i+1}^{n-1}(1-\eta\lambda\sigma_j)\right)
\right]\notag\\
&\leq~
\left(1-\frac{\eta\lambda n}{2}\right)\cdot\E\left[
\left(\sum_{i=0}^{n-1}(1-2\sigma_i)(1-\eta\lambda\sum_{j=i+1}^{n}\sigma_j)\right)\right]
+\frac{301}{100}(\eta\lambda)^2n^3\notag\\
&\stackrel{(*)}{\leq}~\eta\lambda n\left(-\left(1-\frac{\eta\lambda 
n}{2}\right)\frac{n+1}{4(n-1)}+\frac{301}{100}\eta\lambda n^2\right)~,
\end{align*}
where $(*)$ is by \lemref{lem:exp_prod_02}. 
Using the assumptions that $\eta\leq \frac{1}{100\lambda n^2}$ (hence 
$\eta\lambda n\leq \eta\lambda n^2\leq \frac{1}{100}$) and $n\geq 2$, 
this is at most
$-c\eta\lambda n$ for a numerical constant $c>0.2$. Summarizing this part of 
the proof, we have shown that
\begin{equation}\label{eq:bigproduct}
\E\left[\left(\prod_{i=0}^{n-1}(1-\eta\lambda\sigma_i)\right)
\left(\sum_{i=0}^{n-1}(1-2\sigma_i)\prod_{j=i+1}^{n-1}(1-\eta\lambda\sigma_j)\right)
\right]~\leq~ -c\eta\lambda n~.
\end{equation}

Next, we turn to analyze the $\E[x_t]$ term in \eqref{eq:recurse2}. By 
\eqref{eq:recursionform}, and the fact that $\sigma_i$ is independent of $x_t$, 
we have
\[
\E[x_{t+1}]~=~\E[x_t]\cdot\E\left[\prod_{i=0}^{n-1}(1-\eta\lambda\sigma_i)\right]
 +\frac{\eta 
G}{2}\E\left[\sum_{i=0}^{n-1}(1-2\sigma_i)\prod_{j=i+1}^{n-1}(1-\eta\lambda\sigma_j)\right]~.
\]
Again using the notation from \eqref{eq:AB}, 
\lemref{lem:prodtosum}, and the deterministic upper bounds in 
\eqref{eq:detupbound1} and \eqref{eq:detupbound2}, this can be written as
\begin{align*}
\E[x_{t+1}]~&=~\E[x_t]\cdot\E[A]
 +\frac{\eta 
G}{2}\E\left[\sum_{i=0}^{n-1}(1-2\sigma_i)B_i\right]\\
&\leq~
\E[x_t]\cdot\left(\E[\tilde{A}]\pm 
2\left(\eta\lambda\sum_{i=0}^{n-1}\sigma_i\right)^2
\right) +\frac{\eta 
G}{2}\E\left[\sum_{i=0}^{n-1}(1-2\sigma_i)\tilde{B}_i\pm 2\sum_{i=0}^{n-1}
\left(\eta\lambda\sum_{j=i+1}^{n-1}\sigma_j\right)^2\right]\\
&\leq~
\E[x_t]\cdot\left(\left(1-\frac{\eta\lambda 
n}{2}\right)\pm 2(\eta\lambda n)^2
\right) +\frac{\eta 
G}{2}\E\left[\sum_{i=0}^{n-1}(1-2\sigma_i)\tilde{B}_i\pm 2n(\eta\lambda
 n)^2\right].
\end{align*}
Recalling that 
$\E\left[\sum_{i=0}^{n-1}(1-2\sigma_i)\tilde{B}_i\right]=\E\left[
\sum_{i=0}^{n-1}(1-2\sigma_i)(1-\eta\lambda\sum_{j=i+1}^{n}\sigma_j)\right]$ 
and using \lemref{lem:exp_prod_02}, the above is at most
\[
\E\left[x_t\right]\cdot\left(1-\eta\lambda n\left(\frac{1}{2}\pm 
2\eta\lambda n\right)\right)-\frac{\eta^2\lambda n G}{2}
\left(\frac{n+1}{4(n-1)}\pm 2n^2\eta\lambda\right)~.
\]
Using the assumption $\eta\leq \frac{1}{100\lambda n^2}$ and that $n\geq 2$, 
it follows that
\begin{align*}
\E[x_{t+1}]~&\leq~\E\left[x_t\right]\cdot\left(1-\eta\lambda 
n\left(\frac{1}{2}\pm 
\frac{2}{100}\right)\right)-\frac{\eta^2\lambda n 
G}{2}\left(\frac{3}{4}\pm \frac{2}{100}\right)\\
&\leq~
\E[x_t]\cdot\left(1-\eta\lambda n\left(\frac{1}{2}\pm 
\frac{2}{100}\right)\right)-\frac{\eta^2\lambda n 
G}{2}~.
\end{align*}
This inequality implies that if $\E[x_t]\leq 0$, then $\E[x_{t+1}]\leq 0$. 
Since the algorithm is initialized at $x_0=-1$, it follows by induction that 
$\E[x_{t}]\leq 0$ for all $t$, so the inequality above implies that
\[
\E[x_{t+1}]~\leq~\E[x_t]\cdot\left(1-\frac{\eta\lambda 
n}{3}\right)-\frac{\eta^2\lambda n G}{2}~.
\]
Opening the recursion, and using the fact that $x_0=-1$, it follows that
\begin{align*}
\E[x_{t}]~&\leq~-\left(1-\frac{\eta\lambda 
n}{3}\right)^{t}-\frac{\eta^2\lambda n 
G}{2}\sum_{i=0}^{t-1}\left(1-\frac{\eta\lambda n}{3}\right)^i\\
&=~-\left(1-\frac{\eta\lambda 
n}{3}\right)^{t}-\frac{\eta^2\lambda n G}{2(\eta\lambda 
n/3)}\left(1-\left(1-\frac{\eta\lambda n}{3}\right)^{t}\right)\\
&=~
-\left(1-\frac{\eta\lambda 
n}{3}\right)^{t}-\frac{3\eta G}{2}\left(1-\left(1-\frac{\eta\lambda 
n}{3}\right)^{t}\right)~.
\end{align*}
Plugging this and \eqref{eq:bigproduct} into \eqref{eq:recurse2}, we get that
\begin{align*}
\E[x_{t+1}^2]~&\geq~ 
(1-\eta\lambda)^{n}\cdot \E[x_t^2]+\eta G\cdot 
\left(\left(1-\frac{\eta\lambda 
n}{3}\right)^{t}+\frac{3\eta G}{2}\left(1-\left(1-\frac{\eta\lambda 
n}{3}\right)^{t}\right)\right)\cdot c\eta\lambda n\\
&\geq~
(1-\eta\lambda n)\cdot \E[x_t^2]+c\eta^2G\lambda n\cdot 
\left(\left(1-\frac{\eta\lambda 
n}{3}\right)^{t}+\frac{3\eta G}{2}\left(1-\left(1-\frac{\eta\lambda 
n}{3}\right)^{t}\right)\right)~,
\end{align*}
where in the last step we used Bernoulli's inequality. Applying this 
inequality recursively and recalling that $x_0=-1$, it follows that
\begin{equation}\label{eq:bigcomplicated}
\E[x_k^2]~\geq~
(1-\eta\lambda n)^k+c\eta^2 G\lambda n
\sum_{t=0}^{k-1}\left(\left(1-\frac{\eta\lambda 
n}{3}\right)^{t}+\frac{3\eta G}{2}\left(1-\left(1-\frac{\eta\lambda 
n}{3}\right)^{t}\right)\right)\cdot(1-\eta\lambda n)^{k-1-t}
\end{equation}
We now consider two cases:
\begin{itemize}
\item If $\eta\lambda n\leq \frac{1}{2k}$, then \eqref{eq:bigcomplicated} 
implies
\[
\E[x_k^2]~\geq~(1-\eta\lambda 
n)^k~\geq~\left(1-\frac{1}{2k}\right)^k~\geq~\frac{1}{2}
\]
for all $k$.
\item If $\eta\lambda n\geq \frac{1}{2k}$, then \eqref{eq:bigcomplicated} 
implies
\begin{align*}
\E[x_k^2]~&\geq~c\eta^2 G\lambda n
\sum_{t=0}^{k-1}\left(\frac{3\eta G}{2}\left(1-\left(1-\frac{\eta\lambda 
n}{3}\right)^{t}\right)\right)\cdot(1-\eta\lambda n)^{k-1-t}\\
&=~\frac{3c\eta^3G^2\lambda 
n}{2}\sum_{t=0}^{k-1}\left(1-\left(1-\frac{\eta\lambda 
n}{3}\right)^{t}\right)\cdot\left(1-\eta\lambda 
n\right)^{k-1-t}\\
&\geq~\frac{3c\eta^3G^2\lambda 
n}{2}\sum_{t=\lfloor k/2\rfloor}^{k-1}\left(1-\left(1-\frac{\eta\lambda 
n}{3}\right)^{t}\right)\cdot\left(1-\eta\lambda 
n\right)^{k-1-t}\\
&\geq~\frac{3c\eta^3G^2\lambda 
n}{2}\sum_{t=\lfloor k/2\rfloor}^{k-1}\left(1-\left(1-\frac{\eta\lambda 
n}{3}\right)^{\lfloor k/2\rfloor}\right)\cdot\left(1-\eta\lambda 
n\right)^{k-1-t}~.
\end{align*}
Since we assume $\eta\lambda n\geq \frac{1}{2k}$, this is at least
\[
\frac{3c\eta^3G^2\lambda 
n}{2}\sum_{t=\lfloor 
k/2\rfloor}^{k-1}\left(1-\left(1-\frac{1}{6k}\right)^{\lfloor 
k/2\rfloor}\right)\cdot\left(1-\eta\lambda 
n\right)^{k-1-t}~.
\]
Since we assume in the proposition $k>1$, 
$\left(1-\left(1-\frac{1}{6k}\right)^{\lfloor k/2\rfloor}\right)$ can be 
verified to be at least some positive constant $c'>0.04$. Thus, we can lower 
bound the above by
\[
\frac{3cc'\eta^3G^2\lambda 
n}{2}\sum_{t=\lfloor 
k/2\rfloor}^{k-1}\left(1-\eta\lambda 
n\right)^{k-1-t}~=~
\frac{3cc'\eta^3G^2\lambda 
n}{2}\cdot\sum_{t=0}^{k-1-\lfloor k/2\rfloor}\left(1-\eta\lambda 
n\right)^{t}~.
\]
Since $\sum_{i=0}^{r}a^i = \frac{1-a^{r+1}}{1-a}$ for any $a\in (0,1)$ (and 
moreover, $\eta\lambda n\in (0,1)$ by the assumption that $\eta\leq 
\frac{1}{100\lambda n^2}$), the above equals
\[
\frac{3cc'\eta^3G^2\lambda 
n}{2}\cdot\frac{1-(1-\eta\lambda n)^{k-\lfloor k/2\rfloor}}{\eta\lambda n}
~\geq~
\frac{3cc'\eta^2G^2}{2}\cdot\left(1-\left(1-\frac{1}{2k}\right)^{k-\lfloor 
k/2\rfloor}\right)~,
\]
where again we used the assumption $2\eta\lambda n\geq \frac{1}{2k}$. It is 
easily verified that $1-\left(1-\frac{1}{2k}\right)^{k-\lfloor k/2\rfloor}$ is 
lower bounded by a positive constant $>0.2$, so we can lower bound the 
above by $c''(\eta G)^2$ for some numerical constant $c''>0$. Recalling that 
this is a lower bound on $\E[x_k^2]$, and once again using the assumption 
$\eta\lambda n\geq \frac{1}{2k}$, it follows that
\[
\E[x_k^2]~\geq~ c''(\eta G)^2 ~\geq~ c''\left(\frac{G}{4\lambda nk}\right)^2~.
\]
\end{itemize}
Combining the two cases above, we get that there exist some positive numerical 
constant $c'''$ so that
\[
\E[x_k^2]~\geq~ c'''\cdot\min\left\{1~,~\frac{G^2}{\lambda^2 (nk)^2}\right\}~.
\]
Noting that $\E[F(x_k)]=\E[\frac{\lambda}{4}x_k^2]=\frac{\lambda}{4}\E[x_k^2]$ 
and combining with the above, the result follows.

\subsection{Proof of \thmref{thm:single_shuffling}}\label{sec:proofthmsingle}

We will assume without loss of generality that $n$ is even (see the argument at the beginning of the proof of \thmref{thm:main}).

Using the same construction as in the proof of \propref{prop:lowerbound1} (see \eqref{eq:prop1construction}), we begin by observing that our analysis in the first epoch is identical to the random reshuffling case. Therefore, by recursively applying the relation in \eqref{eq:single_epoch} (which in our case makes use of the same permutation in each epoch), we obtain the following relation between the initialization point $ x_0 $ and the $ k $-th epoch $ x_k $
\begin{align*}
	x_k &= (1-\eta\lambda)^{nk}x_0+\frac{\eta G}{2}\sum_{j=0}^{k-1}(1-\eta\lambda)^{nj}\sum_{i=0}^{n-1}\sigma_i(1-\eta\lambda)^i\\
	&= (1-\eta\lambda)^{nk}x_0+\frac{\eta G}{2}\cdot\frac{1-(1-\eta\lambda)^{nk}}{1-(1-\eta\lambda)^n}\sum_{i=0}^{n-1}\sigma_i(1-\eta\lambda)^i.
\end{align*}
From the above, the fact that $ \E[\sigma_i]=0 $, and the assumption $ x_0=1 $ we have
\begin{align*}
	\E[x_k^2] = (1-\eta\lambda)^{2nk} + \left(\frac{\eta G}{2}\right)^2\beta_{n,\eta,\lambda}\left( \frac{1-(1-\eta\lambda)^{nk}}{1-(1-\eta\lambda)^n} \right)^2,
\end{align*}
where $\beta_{n,\eta,\lambda}$ is as defined in \eqref{eq:betadef}.

The remainder of the proof now follows along a similar line as the proof of \propref{prop:lowerbound1}, where we consider different cases based on the value of $ \eta\lambda $.
\begin{itemize}
	\item
	If $ \eta\lambda \ge 1 $, then by \lemref{lem:varterm}, $ 
	\beta_{n,\eta,\lambda} $ is at least some positive constant $c>0$, and also 
	$ \left( \frac{1-(1-\eta\lambda)^{nk}}{1-(1-\eta\lambda)^n} \right)^2\ge1 $ 
	since it is the square of the geometric series 
	$\sum_{j=0}^{k-1}(1-\eta\lambda)^{nj}$ with the first element being equal
	$1$, and the other terms being positive (recall that $ n $ is even). 
	Overall, we get for some constant $ c>0 $ that
	\[
		\E[x_k^2]\ge c\left(\frac{\eta G}{2}\right)^2 \ge 
		\frac{c}{4}\cdot\frac{G^2}{\lambda^2} \ge 
		\frac{c}{4}\cdot\frac{G^2}{\lambda^2nk^2}~.
	\]
	
	\item 
	If $ \eta\lambda\le\frac{1}{nk} $, then
	\[
		\E[x_k^2]\ge (1-\eta\lambda)^{2nk} \ge 
		\left(1-\frac{1}{nk}\right)^{2nk} \ge 
		\left(\frac{1}{4}\right)^2=\frac{1}{16}~.
	\]
	
	\item 
	If $ \eta\lambda\in\left(\frac{1}{nk},\frac{1}{n}\right) $, then by 
	Bernoulli's inequality we have $ \exp(-1/k)\ge(1-\eta\lambda)^n\ge1-n\eta\lambda>0 
	$, implying that
	\[
		\E[x_k^2]\ge \left(\frac{\eta G}{2}\right)^2 \beta_{n,\eta,\lambda} \left(\frac{1-\exp(-1/k)^k}{1-(1-n\eta\lambda)}\right)^2 = \eta^2G^2 \beta_{n,\eta,\lambda} \left(\frac{1-\exp(-1)}{2n\eta\lambda}\right)^2.
	\]
	Using \lemref{lem:varterm} and recalling that $ \eta\lambda\ge\frac{1}{nk} 
	$, we have $ \beta_{n,\eta,\lambda}\ge 
	c\cdot\min\{1+1/\eta\lambda,n^3(\eta\lambda)^2\}\ge cn^3\eta^2\lambda^2 $. 
	Plugging this yields the above is at least
	\[
		c^{\prime}\frac{\eta^4G^2n^3\lambda^2}{n^2\eta^2\lambda^2} = c^{\prime}\eta^2nG^2,
	\]
	for some constant $ c^{\prime} $. Since $ \eta\lambda\ge\frac{1}{nk}\iff \eta\ge\frac{1}{\lambda nk} $, this 
	is lower bounded by
	\[
		c^{\prime}\frac{nG^2}{\lambda^2n^2k^2} = c^{\prime}\frac{G^2}{\lambda^2nk^2}.
	\] 
	\item 
	If $ \eta\lambda\in\left[\frac{1}{n},1\right) $, then 
	recalling $\left(\frac{1-(1-\eta\lambda)^{nk}}{1-(1-\eta\lambda)^{n}}\right)^2\ge1$ as the square of the sum of a geometric series with first element 1 and positive ratio, we have
	\[
		\E[x_k^2]~\geq~ \left(\frac{\eta 
		G}{2}\right)^2\cdot\beta_{n,\eta,\lambda}~.
	\]
	By the assumption on $\eta\lambda$, we have that $ n^3(\eta\lambda)^2\ge1/\eta\lambda $, therefore from \lemref{lem:varterm} the above is at least
	\begin{align*}
	c\left(\frac{\eta 
		G}{2}\right)^2\cdot\min\left\{1+\frac{1}{\eta\lambda}~,~n^3(\eta\lambda)^2\right\}
	~&\ge~c\left(\frac{\eta 
		G}{2}\right)^2\cdot\min\left\{\frac{1}{\eta\lambda}~,~n^3(\eta\lambda)^2\right\}\\
	~&=~c\left(\frac{\eta 
		G}{2}\right)^2\frac{1}{\eta\lambda}~\geq~
	\frac{c\eta G^2}{4\lambda}.~
	\end{align*}
	Since $\eta\geq \frac{1}{\lambda n}$, 
	this is at least
	\[
	\frac{cG^2}{4\lambda^2 n}~\geq~\frac{cG^2}{4\lambda^2 nk^2}~. 
	\]
\end{itemize}
Combining all previous cases, we have that
\[
\E[x_k^2]~\geq~ c\cdot \min\left\{1,\frac{G^2}{\lambda^2 nk^2}\right\}
\]
for some numerical constant $c>0$. Noting that 
$\E[F(x_k)]=\E\left[\frac{\lambda}{2}x_k^2\right]=\frac{\lambda}{2}\E\left[x_k^2\right]$
and combining with the above, the result follows.

\subsection{Proof of \thmref{thm:cyclic}}\label{sec:proofthmcyclic}

We will assume without loss of generality that $n$ is even (see the argument at 
the beginning of the proof of \thmref{thm:main}). 

First, we wish to argue that it is enough to consider the case where $\eta$ is 
such that $\eta\lambda \in (0,1)$:
\begin{itemize}
	\item If $\eta\lambda\geq 2$, it is easy to see 
that the algorithm may not converge. For example, consider the function 
$F(x)=\frac{1}{n}\sum_{i=1}^{n}f_i(x)$ where $f_i(x)=\frac{\lambda}{2}x^2$ for 
all $i$. Then the algorithm performs iterations of the form 
$x_{new}=(1-\eta\lambda)x_{old}$, hence $|x_{new}|\geq |x_{old}|$. Assuming the 
initialization $x_0=1$, we have $F(x_k)=\frac{\lambda}{2}x_k^2\geq 
\frac{\lambda}{2}x_0^2 = \frac{\lambda}{2}$, and the theorem statement holds. 
	\item If $\eta\lambda\in [1,2)$, consider the function 
	$F(x)=\frac{1}{n}\sum_{i=1}^{n}f_i(x)=\frac{\lambda}{2}x^2$ where 
	$f_i(x)=\frac{\lambda}{2}x^2-\frac{G}{2}x$ for odd $i$, and 
	$f_i(x)=\frac{\lambda}{2}x^2+\frac{G}{2}x$ for even $i$, initializing at 
	$x_0=1$. Recalling that $n$ 
	is even, it is easy to verify that
	\[
	x_{t+1}~=~(1-\eta\lambda)^n 
	x_t+\frac{G\eta^2\lambda}{2}\sum_{i=0}^{n/2-1}(1-\eta\lambda)^{2i}~.
	\]
	Since $x_0=1$ and all terms above are non-negative, it follows that 
	$x_k\geq 0$ 
	for all $k\geq 1$. Moreover, since $\eta\geq 1/\lambda$, it follows that 
	$x_{k}\geq \frac{G\eta^2\lambda}{2}\geq \frac{G}{2\lambda}$. Therefore, 
	$F(x_k)=\frac{\lambda}{2}x_k^2\geq \frac{G^2}{8\lambda}\geq 
	\frac{G^2}{8\lambda k^2}$, and the theorem statement holds. 
\end{itemize}

Assuming from now on that $\eta\lambda\in (0,1)$, we turn to our main 
construction. Consider the following function on $\reals$:
\[
F(x)~=~\frac{1}{n}\sum_{i=1}^{n}f_i(x)=\frac{\lambda}{2}x^2~,
\]
where
\[
f_i(x)~=~\begin{cases}\frac{G}{2}x & i\leq \frac{n}{2}\\
\lambda x^2-\frac{G}{2}x & i>\frac{n}{2}\end{cases}~,
\]
Also, we assume that the initialization point $x_0$ is 
$1$.

On this function, we have that during any single epoch, 
we perform $n/2$ iterations of the form
\[
x_{new}~=~x_{old}-\frac{\eta G}{2}~,
\]
followed by $n/2$ iterations of the form
\[
x_{new}~=~(1-\eta\lambda)x_{old}+\frac{\eta G}{2}~.
\]
Thus, after $n$ iterations, we get the following update for a single epoch:
\begin{align}
	x_{t+1} ~&=~ (1-\eta\lambda)^{n/2} \left(x_{t}-\frac{\eta 
		Gn}{4}\right)+\frac{\eta 
		G}{2}\sum_{i=0}^{n/2-1}(1-\eta\lambda)^i\notag\\
	&=~(1-\eta\lambda)^{n/2}x_t+\frac{\eta 
		G}{2}\left(\sum_{i=0}^{n/2-1}(1-\eta\lambda)^i-\frac{n}{2}(1-\eta\lambda)^{n/2}\right)~.
	\label{eq:cyclic1}
\end{align}
Recalling that $\eta\lambda \in (0,1)$, we now consider two cases:
\begin{itemize}
	\item If $\eta\lambda\in (1/n,1)$, we have $\frac{1}{2\eta\lambda}< 
	\frac{n}{2}$. Therefore, 
	\begin{align*}
		\sum_{i=0}^{n/2-1}&(1-\eta\lambda)^i-\frac{n}{2}(1-\eta\lambda)^{n/2}
		~=~\sum_{i=0}^{n/2-1}\left((1-\eta\lambda)^i-(1-\eta\lambda)^{n/2}\right)\\
		&\geq~\sum_{i=0}^{\lceil 1/4\eta\lambda\rceil 
			-1}\left((1-\eta\lambda)^i-(1-\eta\lambda)^{n/2}\right)
		~=~\sum_{i=0}^{\lceil 1/4\eta\lambda\rceil 
			-1}(1-\eta\lambda)^{i}\left(1-(1-\eta\lambda)^{n/2-i}\right)\\
		&\geq~
		\sum_{i=0}^{\lceil 1/4\eta\lambda\rceil 
			-1}(1-\eta\lambda)^{i}\left(1-(1-\eta\lambda)^{1/2\eta\lambda-i}\right)
		~\geq~
		\sum_{i=0}^{\lceil 1/4\eta\lambda\rceil 
			-1}(1-\eta\lambda)^{i}\left(1-(1-\eta\lambda)^{1/4\eta\lambda}\right)~.
	\end{align*}
	Since $1/\eta\lambda \geq 1$, and $(1-1/z)^{z/4}\leq \exp(-1/4)$ for any 
	$z\geq 
	1$, the displayed equation above is at least
	\begin{align*}
		(1-\exp(-1/4))\sum_{i=0}^{\lceil 
			1/4\eta\lambda\rceil 
			-1}(1-\eta\lambda)^{i}~&=~(1-\exp(-1/4))\cdot\frac{1-(1-\eta\lambda)^{\lceil
				1/4\eta\lambda\rceil}}{\eta\lambda}\\
		&\geq~ 
		(1-\exp(-1/4))\cdot\frac{1-(1-\eta\lambda)^{1/4\eta\lambda}}{\eta\lambda}\\
		&\geq~\frac{(1-\exp(-1/4))^2}{\eta\lambda}~.
	\end{align*}
	Denoting $c:=(1-\exp(-1/4))^2>0.04$ and plugging this lower bound on 
	$\sum_{i=0}^{n/2-1}(1-\eta\lambda)^i-\frac{n}{2}(1-\eta\lambda)^{n/2}$ into 
	\eqref{eq:cyclic1}, we get that
	\[
	x_{t+1}~\geq~ (1-\eta\lambda)^{n/2}x_t+\frac{\eta G}{2}\cdot 
	\frac{c}{\eta\lambda},
	\]
	and hence $x_{t+1}\geq \frac{cG}{2\lambda}$. This holds for any $t$, and in 
	particular $x_k\geq \frac{cG}{2\lambda}$, hence 
	$F(x_k)=\frac{\lambda}{2}x_k^2=\frac{c^2 G^2}{8\lambda}\geq \frac{c^2 
		G^2}{8\lambda k^2}$, which satisfies the 
	theorem statement. 
	\item If $\eta\lambda\in (0,1/n]$, we have
	\begin{align*}
		\frac{\eta 
			G}{2}&\left(\sum_{i=0}^{n/2-1}(1-\eta\lambda)^i-\frac{n}{2}(1-\eta\lambda)^{n/2}\right)
		~=~\frac{\eta 
			G}{2}\left(\frac{1-(1-\eta\lambda)^{n/2}}{\eta\lambda}-\frac{n}{2}(1-\eta\lambda)^{n/2}\right)\\
		&=~\frac{G}{2\lambda}\left(1-(1-\eta\lambda)^{n/2}-\frac{\eta\lambda 
			n}{2}(1-\eta\lambda)^{n/2}\right)\\
		&=~\frac{G}{2\lambda}\left(1-\left(1+\frac{\eta\lambda 
			n}{2}\right)(1-\eta\lambda)^{n/2}\right)\\
		&\stackrel{(*)}{\geq}~\frac{G}{2\lambda}\left(1-\left(1+\frac{\eta\lambda
			n}{2}\right)\left(1-\frac{\eta\lambda 
			n}{2}+\frac{(\eta\lambda n/2)^2}{2}\right)\right)\\
		&=~\frac{G}{2\lambda}\left(1-\left(1-\left(\frac{\eta\lambda 
			n}{2}\right)^2+\left(1+\frac{\eta\lambda 
			n}{2}\right)\frac{(\eta\lambda 
			n)^2}{8}\right)\right)\\
		&=~\frac{G(\eta\lambda 
			n)^2}{2\lambda}\left(\frac{1}{4}-\left(1+\frac{\eta\lambda 
			n}{2}\right)\frac{1}{8}\right)~\geq~ 
		\frac{G(\eta\lambda 
			n)^2}{2\lambda}\left(\frac{1}{4}-\left(1+\frac{1}{2}\right)\cdot\frac{1}{8}\right)\\
		&=~ \frac{G\lambda (\eta n)^2}{32}~,
	\end{align*}
	where $(*)$ is by \lemref{lem:binomial}. Plugging this back into 
	\eqref{eq:cyclic1}, we get
	\[
	x_{t+1}~\geq~ (1-\eta\lambda)^{n/2}x_t+\frac{G\lambda (\eta 
		n)^2}{32}~.
	\]
	Recalling that $x_0=1$, this implies that $x_t$ remains positive for all 
	$t$. 
	Also, by Bernoulli's inequality, 
	$1\geq (1-\eta\lambda)^{n/2})\geq 1-\eta\lambda n/2\geq 0$. Therefore, the 
	above 
	displayed 
	equation implies that
	\[
	x_{t+1}~\geq~ \left(1-\frac{\eta\lambda n}{2}\right)x_t+\frac{G\lambda 
	(\eta 
		n)^2}{32}~.
	\]
	Recurseively applying this inequality, and recalling that $x_0=1$, it 
	follows that
	\begin{align*}
	x_{k}~&\geq~\left(1-\frac{\eta\lambda n}{2}\right)^k+\frac{G\lambda(\eta 
	n)^2}{32}\sum_{t=0}^{k-1}\left(1-\frac{\eta\lambda n}{2}\right)^t\\
	&=~\left(1-\frac{\eta\lambda n}{2}\right)^k+\frac{G\lambda(\eta 
		n)^2}{32}\cdot\frac{1-(1-\eta\lambda n/2)^k}{\eta\lambda n/2}\\
	&=~
	\left(1-\frac{\eta\lambda n}{2}\right)^k+
	\frac{G\eta 
		n}{16}\left(1-\left(1-\frac{\eta\lambda n}{2}\right)^k\right)~.
	\end{align*}
	We now consider two sub-cases:
	\begin{itemize}
		\item If $\eta\lambda \in (0,1/nk)$, the above is at least 
		$\left(1-\frac{\eta\lambda 
			n}{2}\right)^k~\geq~\left(1-\frac{1}{2k}\right)^k~\geq~\frac{1}{2}$ 
			for all 
		$k\geq 1$, so we have $F(x_k)=\frac{\lambda}{2}x_k^2\geq 
		\frac{\lambda}{8}$, 
		satisfying the theorem statement. 
		\item If $\eta\lambda \in [1/nk,1/n]$, we have 
		$\left(1-\frac{\eta\lambda 
			n}{2}\right)^k\leq \left(1-\frac{1}{2k}\right)^k\leq 
		\exp\left(-\frac{1}{2}\right)$, so the displayed equation above is at 
		least $\frac{G\eta 
			n}{16}\left(1-\exp\left(-\frac{1}{2}\right)\right)$,
		which by the assumption $\eta\lambda\geq \frac{1}{nk}$, is at least 
		$\frac{1-\exp(-1/2)}{16}\cdot\frac{G}{\lambda k}$. Therefore, 
		\[
		F(x_k)~=~\frac{\lambda}{2}x_k^2~\geq~ \frac{1}{2}\cdot 
		\left(\frac{1-\exp(-1/2)}{16}\right)^2\cdot \frac{G^2}{\lambda k^2}~,
		\]
		which satisfies the theorem statement.
	\end{itemize}
\end{itemize}

\subsection{Proof of \thmref{thm:single_ubound}}\label{sec:proof_thm_single_ubound}

We begin by assuming w.l.o.g.\ that $ b=0 $. This is justified as seen by the transformation $ f_i(x) \mapsto f_i(x-b/\lambda) $ which shifts each $ f_i $ to the right by a distance of $ b/\lambda $, and consequentially shifting the initialization point $ x_0 $ to the right by the same distance to $ x_0 + \frac{b}{\lambda} $. The derivative in the initialization point after transforming remains the same, and a simple inductive argument shows this persists throughout all the iterations of SGD where all the iterates are also shifted by $ b/\lambda $. Additionally, this also entails $ \abs{b_i}\le G $ for all $ i $ since by the gradient boundedness assumption we have $ \abs{a_i x^* - b_i}\le G $ for all $ i $.

Next, we evaluate an expression for the iterate on the $ k $-th epoch $ x_k $. First, for a selected permutation $ \sigma_i:[n]\to[n] $ we have that the gradient update at iteration $ j $ in epoch $ i $ is given by
\[
x_{new} = \p{1-\eta a_{\sigma_i(j)}}x_{old} + \eta b_{\sigma_i(j)}.
\]
Repeatedly applying the above relation, we have that in the end of each epoch the relation between the iterates $ x_t $ and $ x_{t+1} $ is given by
\[
x_{t+1} = \prod_{j=1}^{n}\p{1-\eta a_{\sigma_{t+1}(j)}}x_t + \eta\sum_{j=1}^{n}\p{\prod_{i=j+1}^{n}\p{1-\eta a_{\sigma_{t+1}(i)}}}b_{\sigma_{t+1}(j)}.
\]
Letting $ S \coloneqq \prod_{j=1}^{n}\p{1-\eta a_{\sigma_i(j)}} = \prod_{j=1}^{n}\p{1-\eta a_j} $ and $ X_{\sigma_t}\coloneqq\sum_{j=1}^{n}\p{\prod_{i=j+1}^{n}\p{1-\eta a_{\sigma_{t}(i)}}}b_{\sigma_{t}(j)} $, this can be rewritten equivalently as
\begin{equation}\label{eq:consecutive_iterates}
x_{t+1} = Sx_t +\eta X_{\sigma_{t+1}}.
\end{equation}
Iteratively applying the above, we have after $ t $ epochs that
\begin{equation}\label{eq:t_iterate}
x_t = S^tx_0 + \eta\sum_{i=1}^{t}S^{t-i}X_{\sigma_i}.
\end{equation}
Squaring and taking expectation on both sides yields
\begin{align}\label{eq:pre_S_bound}
\E\pcc{x_k^2} &= \E\pcc{\p{S^kx_0 + \eta\sum_{i=1}^{k}S^{k-i}X_{\sigma_i}}^2} \le 2\E\pcc{S^{2k}x_0^2 + \eta^2\abs{\sum_{i=1}^{k}S^{k-i}X_{\sigma_i}}^2} \nonumber\\ &\le 2S^{2k}x_0^2 + 2\eta^2k\sum_{i=1}^{k}\E\pcc{X_{\sigma_i}^2} = 2S^{2k}x_0^2 + 2\eta^2k^2\E\pcc{X_{\sigma_1}^2},
\end{align}
where the first and second inequalities are application of Jensen's inequality on the function $x\mapsto x^2$ and the last equality is due to the fact that in single shuffling we have $ \sigma_i=\sigma_1 $ for all $ i $.

Since $ \frac{L}{\lambda} \le \frac{nk}{\log\p{n^{0.5}k}} $ implies that $ \eta L\le 1 $, we have $ 1-\eta a_i\in(0,1] $ for any $ i\in\{1,\ldots,n\} $. Using the AM-GM inequality on $ 1-\eta a_1,\ldots,1-\eta a_n $ we have
\begin{equation}\label{eq:amgm}
\sqrt[n]{S}=\sqrt[n]{\prod_{i=1}^{n}(1-\eta a_i)} \le \frac{1}{n}\sum_{i=1}^{n}(1-\eta a_{i})= 1-\frac{\eta\sum_{i=1}^{n}a_i}{n}=1-\eta\lambda,
\end{equation}
implying
\begin{equation}\label{eq:S_bound}
S\le(1-\eta\lambda)^n.
\end{equation}

Recall that $ \eta = \frac{\log\p{n^{0.5}k}}{\lambda nk} $, we combine the above with \lemref{lem:second_moment_ubound} which together with the inequality $ (1 - x/y)^y\le\exp(-x) $ for all $ x,y>0 $ yields that \eqref{eq:pre_S_bound} is upper bounded by
\[
2(1-\eta\lambda)^{2nk}x_0^2 + 2\eta^4n^3k^2G^2L^2 \le \tilde{\Ocal}\p{\frac{1}{nk^2}x_0^2+\frac{G^2L^2}{\lambda^4nk^2}},
\]
and since $ \E\pcc{F(x_k)-F(x^*)} \le \frac{\lambda}{2}\E[x_k^2] $, the theorem follows.

\subsection{Proof of \thmref{thm:reshuffle_ubound}}\label{sec:proof_thm_reshuffling_ubound}

Similarly to the single shuffling case, we assume w.l.o.g.\ that $ b=0 $ and $ \abs{b_i}\le G $ for all $ i\in[n] $ (see the argument in the beginning of the proof of \thmref{thm:single_ubound} for justification). Continuing from \eqref{eq:consecutive_iterates}, we square and take expectation on both sides to obtain
\[
\E\pcc{x_{t+1}^2} = \E\pcc{\p{Sx_t +\eta X_{\sigma_{t+1}}}^2} = S^2\E[x_t^2] + 2\eta S\E\pcc{x_tX_{\sigma_{t+1}}} + \eta^2\E\pcc{X_{\sigma_{t+1}}^2}.
\]
	Since in random reshuffling the random component at iteration $ t+1 $, $ X_{\sigma_{t+1}} $, is independent of the iterate at iteration $ t $, $ x_t $, and by plugging $t=k$ into \eqref{eq:t_iterate}, the above equals
	\begin{align*}
	\E\pcc{x_{t+1}^2} &= S^2\E[x_t^2] + 2\eta S\E\pcc{x_t}\E\pcc{X_{\sigma_{t+1}}} + \eta^2\E\pcc{X_{\sigma_{t+1}}^2}\\
	&= S^2\E[x_t^2] + 2\eta S\E\pcc{S^tx_0 + \eta\sum_{i=1}^{t}S^{t-i}X_{\sigma_i}}\E\pcc{X_{\sigma_{t+1}}} + \eta^2\E\pcc{X_{\sigma_{t+1}}^2} \\
	&= S^2\E[x_t^2] + 2\eta S^{t+1}x_0\E\pcc{X_{\sigma_{t+1}}} + 2\eta^2\sum_{i=1}^{t}S^{t-i+1}\E\pcc{X_{\sigma_i}}\E\pcc{X_{\sigma_{t+1}}} + \eta^2\E\pcc{X_{\sigma_{t+1}}^2} \\
	&= S^2\E[x_t^2] + 2\eta S^{t+1}x_0\E\pcc{X_{\sigma_{1}}} + 2\eta^2\sum_{i=1}^{t}S^{t-i+1}\E\pcc{X_{\sigma_1}}^2 + \eta^2\E\pcc{X_{\sigma_{1}}^2},
	\end{align*}
	where the last equality is due to $ X_{\sigma_i} $ being i.i.d.\ for all $ i $. Recursively applying the above relation and taking absolute value, we obtain
	\begin{equation}
		\E\pcc{x_{k}^2} = S^{2k}x_0^2 + 2\eta x_0\E\pcc{X_{\sigma_1}}\sum_{j=0}^{k-1}S^{k+j} + 2\eta^2\E\pcc{X_{\sigma_1}}^2\sum_{j=0}^{k-1}S^{2j}\sum_{i=1}^{k-j-1}S^{i} + \eta^2\E\pcc{X_{\sigma_{1}}^2}\sum_{j=0}^{k-1}S^{2j},
	\end{equation}
	which by $S\le1$ entails an upper bound of
	\begin{align*}
		\E\pcc{x_{k}^2} &\le S^{2k}x_0^2 + 2\eta \abs{x_0\E\pcc{X_{\sigma_1}}}\sum_{j=0}^{k-1}S^{k+j} + 2\eta^2\E\pcc{X_{\sigma_1}}^2\sum_{j=0}^{k-1}S^{2j}\sum_{i=1}^{k-j-1}S^{i} + \eta^2\E\pcc{X_{\sigma_{1}}^2}\sum_{j=0}^{k-1}S^{2j} \\
		&\le S^{2k}x_0^2 + 2\eta k S^k\abs{x_0}\cdot\abs{\E\pcc{X_{\sigma_1}}} + 2\eta^2k^2\E\pcc{X_{\sigma_1}}^2 + \eta^2k\E\pcc{X_{\sigma_{1}}^2}.
	\end{align*}
	Since $ 2S^k\abs{x_0} \cdot \eta k\abs{\E\pcc{X_{\sigma_1}}} \le S^{2k}x_0^2 + \eta^2k^2\E\pcc{X_{\sigma_1}}^2 $, the above is at most
	\[
	2S^{2k}x_0^2 + 3\eta^2k^2\E\pcc{X_{\sigma_1}}^2 + \eta^2k\E\pcc{X_{\sigma_{1}}^2},
	\]
	and by virtue of \eqref{eq:S_bound}, the inequality $ (1 - x/y)^y\le\exp(-x) $ for all $ x,y>0 $ and Lemmas \ref{lem:second_moment_ubound} and \ref{lem:first_moment_ubound}, we conclude
	\begin{align*}
	\E\pcc{x_{k}^2} &\le 2S^{2k}x_0^2 + 12\eta^4n^2k^2G^2L^2 + 5\eta^4n^3kG^2L^2\log(2n) \\ &\le \tilde{\Ocal}\p{\frac{1}{n^2k^2}x_0^2 + \frac{G^2L^2}{\lambda^4n^2k^2} + \frac{G^2L^2}{\lambda^4nk^3}},
	\end{align*}
	and since $ \E\pcc{F(x_k)-F(x^*)} \le \frac{\lambda}{2}\E[x_k^2] $, the theorem follows.

	\subsection*{Acknowledgements}
	This research is supported in part by European Research Council (ERC) Grant 754705. We thank Shashank Rajput for bringing a mistake in a previous version of \secref{sec:SGDrandom} to our attention.

\bibliographystyle{plain}
\bibliography{bib}

\appendix

\section{Proof of \lemref{lem:varterm}}\label{sec:proofvarterm}

Using \lemref{lem:exp_prod} from Appendix \ref{sec:technical}, we have that
\begin{align}
\E\left[\left(\sum_{i=0}^{n-1}\sigma_i (1-\alpha)^{i}\right)^2\right]~&=~
\E\left[\sum_{i=0}^{n-1}\sum_{j=0}^{n-1}\sigma_i\sigma_j 
(1-\alpha)^{i+j}\right]\notag\\
&=~\sum_{i=0}^{n-1}\E[\sigma_i^2](1-\alpha)^{2i}+\sum_{i,j\in 
\{0,\ldots,n-1\},i\neq j}\E[\sigma_i \sigma_j](1-\alpha)^{i+j}\notag\\
&=~\sum_{i=0}^{n-1}(1-\alpha)^{2i}-\frac{1}{n-1}\left(\left(\sum_{i=0}^{n-1}(1-\alpha)^i\right)^2-
\sum_{i=0}^{n-1} (1-\alpha)^{2i}\right)\notag\\
&=~\left(1+\frac{1}{n-1}\right)\sum_{i=0}^{n-1}(1-\alpha)^{2i}-\frac{1}{n-1}\left(
\sum_{i=0}^{n-1}(1-\alpha)^i\right)^2~.\label{eq:sumnatoanalyze}
\end{align}
Using the fact that $\sum_{i=0}^{r-1}s^i = \frac{1-s^{r}}{1-s}$ for any  $a\neq 
1$, the above can also be written as 
\begin{align}
&\left(1+\frac{1}{n-1}\right)\frac{{1-(1-\alpha)}^{2n}}{1-(1-\alpha)^2}
-\frac{\left(1-(1-\alpha)^n\right)^2}{(n-1)(1-(1-\alpha))^2}\notag\\
&=~
\frac{n}{n-1}\cdot\frac{{1-(1-\alpha)}^{2n}}{\alpha(2-\alpha)}
-\frac{\left(1-(1-\alpha)^n\right)^2}{\alpha^2(n-1)}\notag\\
&=~\frac{n}{n-1}\cdot\frac{1-(1-\alpha)^n}{\alpha(2-\alpha)}\cdot 
\left(1+(1-\alpha)^n-\frac{2-\alpha}{n\alpha}\left(1-(1-\alpha)^n\right)\right)\notag\\
&=~\frac{n}{n-1}\cdot\frac{1-(1-\alpha)^n}{\alpha(2-\alpha)}\cdot 
\left(1-\frac{2-\alpha}{n\alpha}+\left(1+\frac{2-\alpha}{n\alpha 
}\right)(1-\alpha)^n\right)~.\label{eq:natoanalyze}
\end{align}

We now lower bound either \eqref{eq:sumnatoanalyze} or (equivalently) 
\eqref{eq:natoanalyze}, on a case-by-case basis, depending on the size of 
$\alpha$. 

\subsection{The case $\alpha\geq 1$}

We will show that in this case, our equations are lower bounded by a positive 
numerical constant, which satisfies the lemma statement. We split this case 
into a few sub-cases:
\begin{itemize}
\item 
If $\alpha=1$, then \eqref{eq:sumnatoanalyze} equals 
$1+\frac{1}{n-1}-\frac{1}{n-1}=1$.
\item If $\alpha\in (1,2)$, then 
$\frac{2-\alpha}{n\alpha}=\frac{2}{n\alpha}-\frac{1}{n}\leq 
\frac{2}{n}-\frac{1}{n}=\frac{1}{n}$. Using this 
fact, \eqref{eq:natoanalyze} can be lower bounded as
\begin{align*}
2\cdot 
\frac{1-(1-\alpha)^n}{2(2-\alpha)}\cdot\left(1-\frac{2-\alpha}{n\alpha}\right)
~&\geq~
\frac{1-(1-\alpha)^n}{2-\alpha}
\cdot\left(1-\frac{1}{n}\right)\\
&\geq~\frac{1-(1-\alpha)^n}{2(2-\alpha)}\\
&\stackrel{(*)}{=}~\frac{1-|1-\alpha|^n}{2(1-|1-\alpha|)}~\geq~\frac{1-|1-\alpha|}
{2(1-|1-\alpha|)}~=~\frac{1}{2}~,
\end{align*} 
where in $(*)$ we used the facts that $n$ is even and that since $\alpha\in 
(1,2)$, we have 
$2-\alpha=1+1-\alpha=1-|1-\alpha|$.
\item If $\alpha= 2$, then using the assumption that $n$ is even, 
\eqref{eq:sumnatoanalyze} reduces to 
\[
\left(1+\frac{1}{n-1}\right)\sum_{i=0}^{n-1}(-1)^{2i}-\frac{1}{n-1}\left(
\sum_{i=0}^{n-1}(-1)^i\right)^2~=~\left(1+\frac{1}{n-1}\right)n-\frac{1}{n-1}\cdot
 0\geq n~.
\]
\item If $\alpha>2$, then noting that 
$1+\frac{2-\alpha}{n\alpha}=1-\frac{1}{n}+\frac{2}{n\alpha}>0$, 
\eqref{eq:natoanalyze} is lower bounded as
\begin{align*}
2\cdot 
\frac{(1-\alpha)^n-1}{\alpha(\alpha-2)}\cdot\left(1-\frac{2-\alpha}{n\alpha}\right)
~&\geq~
2\cdot\frac{(1-\alpha)^2-1}{\alpha(\alpha-2)}
\cdot\left(1-\frac{2}{n\alpha}+\frac{1}{n}\right)\\
&\geq~2\cdot\left(1-\frac{1}{n}+\frac{1}{n}\right)~=~2~.
\end{align*} 
\end{itemize}

\subsection{The case $\alpha\in [1/13n,1)$}

In this case, we will show a lower bound of $c/\alpha$ for some positive 
numerical constant $c$, which implies the lemma statement in this case.
To show this, we first focus on the term 
\begin{equation}\label{eq:showmonotonic}
1-\frac{2-\alpha}{n\alpha}+\left(1+\frac{2-\alpha}{n\alpha 
}\right)(1-\alpha)^n~,
\end{equation}
in \eqref{eq:natoanalyze}, and argue that it is monotonically increasing in 
$\alpha$. For that, it is enough to show that its derivative with respect to 
$\alpha$ is non-negative. With some straightforward computations, the 
derivative equals
\[
(1-\alpha)^{n-1}\left(1-\frac{2}{\alpha}-n-\frac{2}{\alpha^2 n}+\frac{2}{\alpha 
n}\right)+\frac{2}{\alpha^2 n}~.
\]
this can also be written as
\begin{align}
&\frac{2}{\alpha^2 n}\left((1-\alpha)^{n-1}\left(\frac{\alpha^2 n}{2}-\alpha 
n-\frac{\alpha^2 n^2}{2}-1+\alpha\right)+1\right)\notag\\
&~~~=~\frac{2}{\alpha^2 
n}\left(1-(1-\alpha)^{n-1}\left(1+\alpha(n-1)+\alpha^2\frac{n(n-1)}{2}\right)\right)~.
\label{eq:taylorrem1}
\end{align}
It is easy to verify that $1+\alpha(n-1)+\alpha^2\frac{n(n-1)}{2}$ is the 
third-order Taylor expansion of the function $g(\alpha):=(1-\alpha)^{1-n}$ 
around $\alpha=0$, and moreover, it is a lower bound on the function (for 
$\alpha\in [1/13n,1)$) since the Taylor remainder term (in Lagrange form) 
equals 
$\frac{g^{(3)}(\xi)}{3!}\alpha^3=\frac{(n-1)n(n+1)}{3!(1-\xi)^{n+2}}\alpha^3$
for some $\xi\in [0,\alpha]$, which is strictly positive for any $\alpha$ in 
our range. Overall, we can lower bound \eqref{eq:taylorrem1} by 
\[
\frac{2}{\alpha^2 n}\left(1-(1-\alpha)^{n-1}\cdot (1-\alpha)^{1-n}\right)~=~0.
\]
This implies that \eqref{eq:showmonotonic} is monotonically increasing. 

Using this monotonicity property, we get that \eqref{eq:showmonotonic} is 
minimized over the 
interval $\alpha\in [1/13n,1)$ when $\alpha=1/13n$, in which case it takes the 
value
\begin{align*}
1-\left(26-\frac{1}{n}\right)+\left(1+26-\frac{1}{n}\right)\left(1-\frac{1}{13n}\right)^n
~&=~\left(27-\frac{1}{n}\right)\left(1-\frac{1}{13n}\right)^n+\frac{1}{n}-25\\
&=~27\left(1-\frac{1}{13n}\right)^n+\frac{1}{n}\left(1-\left(1-\frac{1}{13n}\right)^n\right)-25~.
\end{align*}
A numerical computation reveals that this expression is strictly positive 
(lower bounded by $7\cdot 10^{-4}$) for all $2\leq n< 78$. For $n\geq 
78$, noting that $(1-1/13n)^n$ is monotonically increasing in $n$, this 
expression can be lower bounded by 
\[
27\left(1-\frac{1}{13n}\right)^n-25~\geq~ 27\left(1-\frac{1}{13\cdot 
78}\right)^{78}-25~>~2\cdot 10^{-7}.
\]
In any case, we get that 
\eqref{eq:showmonotonic} is lower bounded by some positive numerical constant 
$c$. 
Plugging back into \eqref{eq:natoanalyze}, and using that fact that 
$(1-1/13n)^n$ is upper bounded by $\exp(-1/13)$, we can lower bound that 
equation by
\[
\frac{n}{n-1}\cdot \frac{1-(1-\alpha)^n}{\alpha(2-\alpha)}\cdot c~\geq~
c\cdot\frac{1-(1-1/13n)^n}{2\alpha}~\geq~c\cdot 
\frac{1-\exp(-1/13)}{2\alpha}~,
\]
which equals $c'/\alpha$ for some numerical constant $c'>0$. 

\subsection{The case $\alpha\in (0,1/13n)$}

In this case, we have $n^3\alpha^2\leq \frac{1}{\alpha}$, so it is enough to 
prove a lower bound of $c\cdot n^3\alpha^2$ in 
order to satisfy the lemma statement. We analyze seperately the cases $n=2$ and 
$n>2$. 
If $n=2$, then \eqref{eq:natoanalyze} equals
\begin{align*}
&2\cdot 
1\cdot\left(1-\frac{2-\alpha}{2\alpha}+\left(1+\frac{2-\alpha}{2\alpha}\right)(1-\alpha)^2\right)\\
&=~2\left(\frac{3}{2}-\frac{1}{\alpha}+\left(\frac{1}{2}+\frac{1}{\alpha}\right)(1-\alpha)^2\right)\\
&=~2\left(2-2\alpha\left(\frac{1}{2}+\frac{1}{\alpha}\right)+\alpha^2\left(\frac{1}{2}+\frac{1}{\alpha}\right)\right)
~=~\alpha^2~=~\frac{1}{8}n^3\alpha^2~,
\end{align*}
which satisfies the lower bound in the lemma statement. If $n>2$, by 
\lemref{lem:binomial}, \eqref{eq:natoanalyze} equals
\[
\frac{n}{n-1}\cdot\frac{1-(1-\alpha)^n}{\alpha(2-\alpha)}\cdot 
\left(1-\frac{2-\alpha}{\alpha n}+\left(1+\frac{2-\alpha}{\alpha n
}\right)\left(1-\alpha n+\binom{n}{2}\alpha^2-\binom{n}{3}\alpha 
^3+c_{\alpha,n}\right)\right)~,
\]
where $|c_{\alpha,n}|\leq (\alpha n)^4/24$. Simplifying a bit, 
this equals
\begin{align*}
&\frac{n}{n-1}\cdot\frac{1-(1-\alpha)^n}{\alpha(2-\alpha)}\cdot 
\left(2+\left(1+\frac{2-\alpha}{\alpha n
}\right)\left(-\alpha 
n+\binom{n}{2}\alpha^2-\binom{n}{3}\alpha^3+c_{\alpha,n}\right)\right)\\
&=~
\frac{n}{n-1}\cdot\frac{1-(1-\alpha)^n}{\alpha(2-\alpha)}\cdot 
\left(2+\left(1-\frac{1}{n}+\frac{2}{\alpha n}\right)\left(-\alpha 
n+\binom{n}{2}\alpha^2-\binom{n}{3}\alpha^3+c_{\alpha,n}\right)\right)~.\\
\end{align*}
Opening the inner product and collecting terms according to powers of $\alpha$, 
this equals
\begin{align}
\frac{n}{n-1}\cdot\frac{1-(1-\alpha)^n}{\alpha(2-\alpha)}\cdot 
&\left(\left(
-n+1+\frac{2}{n}\binom{n}{2}\right)\alpha+\right.\notag\\
&~~\left.\left(\left(1-\frac{1}{n}\right)\binom{n}{2}
-\frac{2}{n}\binom{n}{3}\right)\alpha^2+\left(1-\frac{1}{n}\right)\binom{n}{3}\alpha^3+\left(1-\frac{1}{n}+\frac{2}{\alpha
 n}\right)c_{\alpha,n}\right)~.\label{eq:bigterm}
\end{align}
It is easily verified that 
\[
-n+1+\frac{2}{n}\binom{n}{2}=0~~~,~~~
\left(1-\frac{1}{n}\right)\binom{n}{2}
-\frac{2}{n}\binom{n}{3}\geq\frac{(n-1)^2}{6}~~~,~~~
\left(1-\frac{1}{n}\right)\binom{n}{3}\leq n^3~.
\]
Plugging this into \eqref{eq:bigterm}, and recalling that $|c_{\alpha,n}|\leq 
(\alpha n)^4/24$, we can lower bound \eqref{eq:bigterm} by
\begin{equation*}
\frac{n}{n-1}\cdot\frac{1-(1-\alpha)^n}{\alpha(2-\alpha)}\cdot 
\left(\frac{(n-1)^2}{6}\alpha^2-(\alpha n)^3-\left|1-\frac{1}{n}+\frac{2}{\alpha
	n}\right|\frac{(\alpha n)^4}{24}\right)~.
\label{eq:113n}
\end{equation*}
Invoking again \lemref{lem:binomial}, and noting that $n/(n-1)\geq 1$ and 
$\alpha n\in (0,1/13)$, we can lower bound the above by
\begin{align*}
&1\cdot \frac{1-(1-\alpha n+(\alpha n)^2/2)}{\alpha(2-\alpha)}\left(\frac{(n-1)^2}{6}\alpha^2-(\alpha n)^3-\frac{3}{\alpha n}\cdot\frac{(\alpha n)^4}{24}\right)\\
&~=~
\frac{\alpha n(1-\alpha n/2)}{\alpha(2-\alpha)}\left(\frac{(n-1)^2}{6}\alpha^2-\frac{9}{8}(\alpha n)^3\right)\\
&\geq~
\frac{n}{2(2-\alpha)}\left(\frac{(n-1)^2}{6}\alpha^2-\frac{9}{8}(\alpha n)^3\right)\\
&\geq~
\frac{n^3\alpha^2}{4}\left(\frac{(n-1)^2}{6n^2}-\frac{9}{8}\alpha n\right)
~=~\frac{n^3\alpha^2}{4}\left(\frac{1}{6}\left(1-\frac{1}{n}\right)^2-\frac{9}{8}\alpha
 n\right)~.
\end{align*}
Since we can assume $n\geq 4$ (as $n$ is even and the case $n=2$ was treated 
earlier), and $\alpha n \leq 1/13$, it can be easily verified that 
this is at least $c n^3\alpha^2$ for some positive constant $c>10^{-3}$.

\section{Technical Lemmas}\label{sec:technical}

\begin{lemma}\label{lem:exp_prod}
Let $\sigma_0,\ldots,\sigma_{n-1}$ be a random permutation of 
$(1,...,1,-1,...,-1)$ (where there are $n/2$ $1$'s and $n/2$ $-1$'s). Then for 
any indices $i,j$,
\[
\E[\sigma_i \sigma_j]~=~\begin{cases} 1& 
\text{if}~~i=j\\-\frac{1}{n-1}&\text{if}~~i\neq j\end{cases}~.
\]
\end{lemma}
\begin{proof}
Note that each $\sigma_i$ is uniformly distributed on $\{-1,+1\}$. Therefore, 
$\E[\sigma_i^2]=1$, and for any $i\neq j$,
\begin{align*}
\E[\sigma_i \sigma_j] ~&=~ 
\frac{1}{2}\E[\sigma_i|\sigma_j=1]-\frac{1}{2}\E[\sigma_i|\sigma_j=-1]\\
&=~
\frac{1}{2}\left(\Pr(\sigma_i=1|\sigma_j=1)-\Pr(\sigma_i=-1|\sigma_j=1)
-\Pr(\sigma_i=1|\sigma_j=-1)+\Pr(\sigma_i=-1|\sigma_j=-1)\right)\\
&=~ 
\frac{1}{2}\left(\frac{n/2-1}{n-1}-\frac{n/2}{n-1}-\frac{n/2}{n-1}+\frac{n/2-1}{n-1}\right)~=~-\frac{1}{n-1}~.
\end{align*}
\end{proof}

\begin{lemma}\label{lem:exp_prod_01}
Let $\sigma_0,\ldots,\sigma_{n-1}$ be a random permutation of 
$(1,...,1,0,...,0)$ (where there are $n/2$ $1$'s and $n/2$ $0$'s). Then for 
any indices $i,j$,
\[
\E[\sigma_i \sigma_j]~=~\begin{cases} \frac{1}{2}& 
\text{if}~~i=j\\\frac{1}{4}\left(1-\frac{1}{n-1}\right)&\text{if}~~i\neq 
j\end{cases}~.
\]
\end{lemma}
\begin{proof}
This follows from applying \lemref{lem:exp_prod} on the random variables 
$\mu_0,\ldots,\mu_{n-1}$, where $\mu_i:=1-2\sigma_i$ for all $i$, and noting that 
$\E[\mu_i\mu_j]=\E[(1-2\sigma_i)(1-2\sigma_j)]=4\E[\sigma_i \sigma_j]-1$ (using 
the fact that each $\sigma_i$ is uniform on $\{0,1\}$). 
\end{proof}

\begin{lemma}\label{lem:exp_prod_02}
Under the conditions of \lemref{lem:exp_prod_01}, we have that
\[
\E\left[
\left(\sum_{i=0}^{n-1}(1-2\sigma_i)(1-\eta\lambda\sum_{j=i+1}^{n}\sigma_j)\right)\right]
~=~
-\eta\lambda \frac{n(n+1)}{4(n-1)}
\]
\end{lemma}
\begin{proof}
Using \lemref{lem:exp_prod_01}, and the fact that each $\sigma_i$ is uniform on 
$\{0,1\}$, we have
\begin{align*}
\E&\left[
\left(\sum_{i=0}^{n-1}(1-2\sigma_i)(1-\eta\lambda\sum_{j=i+1}^{n}\sigma_j)\right)\right]\\
&=~\E\left[n-2\sum_{i=0}^{n-1}\sigma_i-\eta\lambda\sum_{i=0}^{n-1}\sum_{j=i+1}^{n}\sigma_j
+2\eta\lambda\sum_{i=0}^{n-1}\sum_{j=i+1}^{n}\sigma_i\sigma_j\right]\\
&=~n-n-\eta\lambda\cdot \frac{n(n+1)}{2}\cdot\frac{1}{2}+
2\eta\lambda\cdot\frac{n(n+1)}{2}\cdot\frac{1}{4}\left(1-\frac{1}{n-1}\right)\\
&-\eta\lambda\cdot\frac{n(n+1)}{4}+\eta\lambda\cdot\frac{n(n+1)}{4}
\left(1-\frac{1}{n-1}\right)\\
&~=-\eta\lambda \frac{n(n+1)}{4(n-1)}~.
\end{align*}
\end{proof}

\begin{lemma}\label{lem:binomial}
	Let $r$ be a positive integer and $x\in [0,1]$. Then for any positive 
	integer $j<r$,
	\[
	(1-x)^n~=~\sum_{i=0}^{j}(-1)^i\binom{r}{i}x^i+a_{j,x}~,
	\]
	where $\binom{r}{1},\binom{r}{2}$ etc. refer to binomial coefficients, and 
	$a_{j,x}$ has the same sign as $(-1)^{j+1}$ and satisfies 
	\[
	|a_{j,x}|~\leq ~
	\frac{(rx)^{j+1}}{(j+1)!}~.
	\]
\end{lemma}
\begin{proof}
	The proof follows by a Taylor expansion of the function $g(x)=(1-x)^r$ 
	around 
	$x=0$: It is easily verified that the first $j$ terms are 
	$\sum_{i=0}^{j}(-1)^i\binom{r}{i}x^i$. Moreover, by Taylor's theorem, the 
	remainder term $\alpha_{j,x}$ (in Lagrange form) is 
	$\frac{g^{(j+1)}(\xi)}{(j+1)!}x^{j+1}$ 
	for some $\xi\in [0,x]$. Moreover, $g^{(j+1)}(\xi) = 
	(-1)^{j+1}\binom{r}{j+1}(1-\xi)^{r-j-1}$, whose sign is $(-1)^{j+1}$ and 
	absolute value at most
	\[
	\sup_{\xi\in [0,x]}\binom{r}{j+1}(1-\xi)^{r-j-1} 
	x^{j+1}\leq \frac{r^{j+1}}{(j+1)!}\cdot 1\cdot x^{j+1}~.
	\]
\end{proof}

\begin{lemma}\label{lem:prodtosum}
Let $a_1,\ldots,a_n$ be a sequence of elements in 
$\left[0,\frac{1}{10n}\right]$. 
Then
\[
\left|\prod_{i=1}^{n}(1-a_i)-\left(1-\sum_{i=1}^{n}a_i\right)\right|~\leq~
2\left(\sum_{i=1}^{n}a_i\right)^2~.
\]
\end{lemma}
\begin{proof}
We have $\prod_{i=1}^{n}(1-a_i)=\exp\left(\sum_{i=1}^{n}\log(1-a_i)\right)$. By 
a standard Taylor expansion of $\log(1-x)$ around $x=0$, we have for any 
$a_i\in [0,1/10n]$ 
\[
|\log(1-a_i)+a_i|~\leq~ 
\frac{a_i^2}{2(1-a_i)^2}~\leq~\frac{1}{2(9/10)^2}a_i^2~\leq~
\frac{5}{8}a_i^2~.
\]
In particular, this implies that
\begin{equation}\label{eq:logtolin}
\left|\sum_{i=1}^{n}\log(1-a_i)+\sum_{i=1}^{n}a_i\right|\leq 
\frac{5}{8}\sum_{i=1}^{n}a_i^2~.
\end{equation}
Since $a_i\in [0,1/10n]$, this means that
\[
\left|\sum_{i=1}^{n}\log(1-a_i)\right|~\leq~\sum_{i=1}^{n}a_i+\frac{5}{8}\sum_{i=1}^{n}a_i^2
~\leq~ \frac{1}{10}+\frac{5}{8\cdot 100n}~<~\frac{1}{9}~.
\]
Using the above two inequalities, and a Taylor expansion of $\exp(x)$ around 
$x=0$, we have
\begin{align*}
\left|\exp\left(\sum_{i=1}^{n}\log(1-a_i)\right)-\left(1+\sum_{i=1}^{n}\log(1-a_i)\right)\right|
~&\leq~ \max_{\xi\in [\sum_i \log(1-a_i),0]} 
\frac{\exp(\xi)}{2}\left(\sum_{i=1}^{n}\log(1-a_i)\right)^2\\
&\leq~ 
\frac{1}{2}\left(\sum_{i=1}^{n}a_i+\frac{5}{8}\sum_{i=1}^{n}a_i^2\right)^2\\
&\leq~
\frac{1}{2}\left(\frac{13}{8}\sum_{i=1}^{n}a_i\right)^2~.
\end{align*}
Combining this with \eqref{eq:logtolin}, and using the fact that 
$\exp(\sum_i\log(1-a_i))=\prod_i (1-a_i)$, we get that
\[
\left|\prod_{i=1}^{n}(1-a_i)-\left(1-\sum_{i=1}^{n}a_i\right)\right|
~\leq~ \frac{5}{8}\sum_{i=1}^{n}a_i^2 
+\frac{1}{2}\left(\frac{13}{8}\sum_{i=1}^{n}a_i\right)^2~.
\]
Simplifying, the result follows.
\end{proof}

\begin{lemma}\label{lem:second_moment_ubound}
	Let $ X_{\sigma}\coloneqq\sum_{j=1}^{n}\p{\prod_{i=j+1}^{n}\p{1-\eta a_{\sigma(i)}}}b_{\sigma(j)} $ where each $ f_i(x)=\frac{a_i}{2}x^2 + b_ix $ satisfies Assumption \ref{assumption_ubound}, $ \sum_{i=1}^{n}b_i=0 $ and $ \eta L\le1 $. Then
	\[
		\E_{\sigma}\pcc{X_{\sigma}^2} \le 5\eta^2n^3L^2 G^2\log(2n),
	\]
	where the expectation is over sampling a permutation $ \sigma:[n]\to[n] $ uniformly at random.
\end{lemma}

\begin{proof}
	Using summation by parts on $ \alpha_j=\prod_{i=j+1}^{n}\p{1-\eta a_{\sigma(i)}} $ and $ \beta_j = b_{\sigma(j)} $, we have
	\begin{align}
		X_{\sigma}^2 &= \p{\sum_{j=1}^{n}\p{\prod_{i=j+1}^{n}\p{1-\eta a_{\sigma(i)}}}b_{\sigma(j)}}^2\nonumber\\ &= \p{\sum_{j=1}^{n}b_{\sigma(j)} - \sum_{j=1}^{n-1}\p{\prod_{i=j+2}^{n}\p{1-\eta a_{\sigma(i)}} - \prod_{i=j+1}^{n}\p{1-\eta a_{\sigma(i)}}} \sum_{i=1}^{j} b_{\sigma(i)}}^2 \nonumber\\
		&= \p{\eta\sum_{j=1}^{n-1} a_{\sigma(j+1)} \prod_{i=j+2}^{n}\p{1-\eta a_{\sigma(i)}} \sum_{i=1}^{j} b_{\sigma(i)}}^2\nonumber\\ &\le \p{\eta L \sum_{j=1}^{n-1}\abs{\sum_{i=1}^{j} b_{\sigma(i)}}}^2 \le \eta^2n^2L^2\p{\sum_{j=1}^{n-1}\abs{\frac1j\sum_{i=1}^{j} b_{\sigma(i)}}}^2, \label{eq:pre_hoeffding}
	\end{align}
	where the first inequality is due to $ 0\le a_i\le L $ for all $ i $ and $ \eta L\le1 $ which implies $ 1-\eta a_{\sigma(i)} \in [0,1] $ for all $ i $. Next, without any assumptions on $ \sigma $ we derive a worst-case bound. Since $ \abs{b_i}\le G $ for all $ i $, we have
	\begin{equation}\label{eq:worst_case}
		X_{\sigma}^2 \le \eta^2n^4G^2L^2.
	\end{equation}
	The above worst-case bound can be used to show a $ \tilde{\Ocal}(1/k^2) $ upper bound on the sub-optimality of the incremental gradient method which accords with known results (see Table \ref{table:results}). However, a more careful examination of the random sum reveals that when choosing $ \sigma $ uniformly at random, a concentration of measure phenomenon occurs which allows us to establish the stronger bound in the lemma (with linear dependence rather than quadratic in $ n $), and improve the sub-optimality. We use the following version of the Hoeffding-Serfling inequality \cite[Corollary 2.5]{bardenet2015concentration}, stated here for completeness.
	\begin{theorem}[Hoeffding-Serfling inequality]\label{thm:hoeffding_serfling}
		Suppose $ n\ge2 $, $ x_1,\ldots,x_n\in[a,b] $ with mean $ \bar{x} $ and $ \sigma:[n]\to[n] $ is a permutation sampled uniformly at random. Then for all $ j\le n $, for all $ \delta\in[0,1] $, w.p.\ at least $ 1-\delta $ it holds that
		\[
		\frac{1}{j}\sum_{i=1}^{j}\p{x_{\sigma(i)}-\bar{x}} \le (b-a)\sqrt{\frac{\rho_j\log(1/\delta)}{2j}},
		\]
		where
		\[
		\rho_j=\min\set{1-\frac{j-1}{n}, \p{1-\frac{j}{n}}\p{1+\frac1j}}.
		\]
	\end{theorem}
	Since $ \rho_j\le1 $ for all $ j\in[n] $ and by applying the inequality on $ -x_1,\ldots,-x_n $ and using the union bound, we have w.p.\ at least $ 1-\delta $ that
	\[	
	\abs{\frac{1}{j}\sum_{i=1}^{j}\p{x_{\sigma(i)}-\bar{x}}} \le (b-a)\sqrt{\frac{\log(2/\delta)}{2j}}.
	\]
	Using the union bound again for the $ n $ events where each of the $ n $ partial sums do not deviate, we have
	\begin{align}	
	\sum_{j=1}^{n}\abs{\frac{1}{j}\sum_{i=1}^{j}\p{x_{\sigma(i)}-\bar{x}}} &\le (b-a)\sqrt{\frac{\log(2n/\delta)}{2}}\sum_{j=1}^{n}\frac{1}{\sqrt{j}} \le (b-a)\sqrt{\frac{\log(2n/\delta)}{2}}\p{1 + \int_{2}^{n}\frac{1}{\sqrt{x-1}}dx}\nonumber\\ &=(b-a)\sqrt{\frac{\log(2n/\delta)}{2}}(2\sqrt{n-1}-1)\ \le 2(b-a)\sqrt{n\log(2n/\delta)}.\nonumber
	\end{align}
	Using the above to bound \eqref{eq:pre_hoeffding} w.h.p.\ we have that w.p.\ at least $ 1-\delta $
	\begin{equation*}
		X_{\sigma}^2 \le \eta^2n^2L^2\cdot 2G^2n\log(2n/\delta) = 2\eta^2n^3G^2L^2\log(2n/\delta).
	\end{equation*}
	Letting $ \delta = \frac1n $, we denote the event where $ X_{\sigma}^2 \le 4\eta^2n^3G^2L^2\log(2n) $ as $ E $, and we have that the complement of $ E $ satisfies $ \Pr\pcc{\bar{E}}\le\frac{1}{n} $ and
	\begin{equation*}
	\E\pcc{X_{\sigma}^2|E}\le 4\eta^2n^3G^2L^2\log(2n).
	\end{equation*}
	Finally, from the above, the law of total expectation and \eqref{eq:worst_case} we have
	\begin{align*}
		\E\pcc{X_{\sigma}^2} &= \E\pcc{X_{\sigma}^2|E}\Pr\pcc{E} + \E\pcc{X_{\sigma}^2|\bar{E}}\Pr\pcc{\bar{E}} \\ &\le 4\eta^2n^3G^2L^2\log(2n)\cdot1 + \eta^2n^4G^2L^2 \cdot \frac{1}{n} \\ &\le 5\eta^2n^3G^2L^2\log(2n).
	\end{align*}
	
\end{proof}

\begin{lemma}\label{lem:first_moment_ubound}
	Let $ X_{\sigma}\coloneqq\sum_{j=1}^{n}\p{\prod_{i=j+1}^{n}\p{1-\eta a_{\sigma(i)}}}b_{\sigma(j)} $ where each $ f_i(x)=\frac{a_i}{2}x^2 + b_ix $ satisfies Assumption \ref{assumption_ubound}, $ \sum_{i=1}^{n}b_i=0 $ and $ \eta nL\le 0.5 $. Then
	\[
	\abs{\E_{\sigma}\pcc{X_{\sigma}}} \le 2\eta nGL,
	\]
	where the expectation is over sampling a permutation $ \sigma:[n]\to[n] $ uniformly at random.
\end{lemma}

\begin{proof}
	Letting $ Y_j \coloneqq \p{\prod_{i=j+1}^{n}\p{1-\eta a_{\sigma(i)}}}b_{\sigma(j)} $, we expand $ Y_j $ to obtain
	\begin{align}
		\E\pcc{Y_j} &= 
		\E\pcc{b_{\sigma(j)}} + \sum_{m=1}^{n-j}(-\eta)^m \E\pcc{\sum_{j+1\le i_1,\ldots,i_m\le n\text{ distinct}}~\p{\prod_{l=1}^{m}a_{\sigma(i_l)}}b_{\sigma(j)}} \nonumber\\
		&=\sum_{m=1}^{n-j}(-\eta)^m \E\pcc{\sum_{j+1\le i_1,\ldots,i_m\le n\text{ distinct}}~\p{\prod_{l=1}^{m}a_{\sigma(i_l)}}b_{\sigma(j)}}\label{eq:expectation_term}
	\end{align}
	Repeatedly using the law of total expectation, the expectation term in the right hand side above equals
	\begin{align}
		& \sum_{t_1\in[n]} \E\pcc{\sum_{j+1\le i_1,\ldots,i_m\le n\text{ distinct}}~\p{\prod_{l=1}^{m}a_{\sigma(i_l)}}b_{\sigma(j)}\Bigg\vert \sigma(i_1)=t_1}\Pr\pcc{\sigma(i_1)=t_1} \nonumber\\
		=&  \frac1n\sum_{t_1\in[n]} a_{t_1}\E\pcc{\sum_{j+1\le i_2,\ldots,i_m\le n\text{ distinct}}~\p{\prod_{l=2}^{m}a_{\sigma(i_l)}}b_{\sigma(j)}\Bigg\vert \sigma(i_1)=t_1} \nonumber\\
		=& \frac{1}{n(n-1)}\sum_{t_1\in[n]}~ \sum_{t_2\in[n]\setminus\set{t_1}} a_{t_1}a_{t_2}\E\pcc{\sum_{j+1\le i_3,\ldots,i_m\le n\text{ distinct}}~\p{\prod_{l=3}^{m}a_{\sigma(i_l)}}b_{\sigma(j)}\Bigg\vert \sigma(i_1)=t_1,\sigma(i_2)=t_2}\nonumber\\
		=&\ldots\nonumber\\
		=&\frac{(n-m)!}{n!} \sum_{t_1\in[n]}~ \sum_{t_2\in[n]\setminus\set{t_1}} \ldots \sum_{t_m\in[n]\setminus\set{t_1,\ldots,t_{m-1}}} a_{t_1}a_{t_2}\ldots a_{t_m}\E\pcc{b_{\sigma(j)}\Bigg\vert \sigma(i_1)=t_1,\ldots,\sigma(i_m)=t_m} \nonumber\\
		=& \frac{(n-m)!}{n!} \sum_{t_1\in[n]}~ \sum_{t_2\in[n]\setminus\set{t_1}} \ldots \sum_{t_m\in[n]\setminus\set{t_1,\ldots,t_{m-1}}}a_{t_1}a_{t_2}\ldots a_{t_m}
		~\frac{1}{n-m}\sum_{t_{m+1}\in[n]\setminus\set{t_1,\ldots,t_m}} b_{t_{m+1}} \nonumber\\
		=& -\frac{(n-m)!}{n!} \sum_{t_1\in[n]}~ \sum_{t_2\in[n]\setminus\set{t_1}} \ldots \sum_{t_m\in[n]\setminus\set{t_1,\ldots,t_{m-1}}}
		a_{t_1}a_{t_2}\ldots a_{t_m}~\frac{1}{n-m}\sum_{t_{m+1}\in\set{t_1,\ldots,t_m}} b_{t_{m+1}}.\label{eq:maclaurin}
	\end{align}
	Recalling that $ \abs{a_i}\le L $ and $ \abs{b_i}\le G $, the above is upper bounded in absolute value by.
	\[
	\frac{(n-m)!}{n!} \sum_{t_1\in[n]}~ \sum_{t_2\in[n]\setminus\set{t_1}} \ldots \sum_{t_m\in[n]\setminus\set{t_1,\ldots,t_{m-1}}}
	L^m~\frac{1}{n-m}\sum_{t_{m+1}\in\set{t_1,\ldots,t_m}} G \le \frac{m}{n-m}L^mG.
	\]
	Plugging this back in \eqref{eq:expectation_term} we obtain
	\begin{align*}
		\abs{\E\pcc{Y_j}} &\le \sum_{m=1}^{n-j}\abs{(-\eta)^m \frac{m}{n-m}L^mG} \le \sum_{m=1}^{n-1}\eta^m\frac{m}{n-m}L^mG \\
		&\le \sum_{m=1}^{n-1}\eta^mn^{m-1}L^mG \le G\sum_{m=1}^{\infty}\eta^m n^{m-1}L^m \\
		&\le G\frac{\eta L}{1-\eta nL} \le 2\eta GL.
	\end{align*}
	Where the last two inequalities are by the assumption $ \eta nL \le 0.5 $ which guarantees that the sum converges. Finally, we conclude
	\[
		\abs{\E\pcc{X_{\sigma}}} \le \sum_{j=1}^{n}\abs{\E\pcc{Y_j}} \le 2\eta nGL.
	\]
\end{proof}

\end{document}